\documentclass{article}

\usepackage{microtype}
\usepackage{graphicx}
\usepackage{subcaption}
\usepackage{booktabs} 
\usepackage{graphicx} 
\usepackage{marvosym}

\usepackage{hyperref}

\usepackage[preprint]{icml2026}

\usepackage{amsmath}
\usepackage{amssymb}
\usepackage{mathtools}
\usepackage{amsthm}

\usepackage[capitalize,noabbrev]{cleveref}

\theoremstyle{plain}
\newtheorem{theorem}{Theorem}[section]

\newtheorem{lemma}[theorem]{Lemma}

\theoremstyle{definition}

\theoremstyle{remark}

\newcommand{\Framework}{DEER}

\def\aka{\emph{a.k.a.,~}}

\usepackage[textsize=tiny]{todonotes}
\usepackage{threeparttable}
\usepackage[table]{xcolor}
\usepackage{multirow}
\usepackage{tcolorbox}
\usepackage{listings}
\definecolor{prompt}{RGB}{59, 130, 246}      
\definecolor{iter0}{RGB}{52, 211, 153}       
\definecolor{iter1}{RGB}{251, 191, 36}       
\definecolor{iter2}{RGB}{168, 85, 247}       
\definecolor{iter3}{RGB}{239, 68, 68}        
\usepackage{graphicx}
\usepackage{subcaption} 
\usepackage[T1]{fontenc}
\usepackage{newtxtext}
\usepackage{newtxmath}

\icmltitlerunning{Submission and Formatting Instructions for ICML 2026}

\begin{document}

\twocolumn[
  \icmltitle{
\texorpdfstring{
  \raisebox{-0.25\height}{\includegraphics[width=0.04\textwidth]{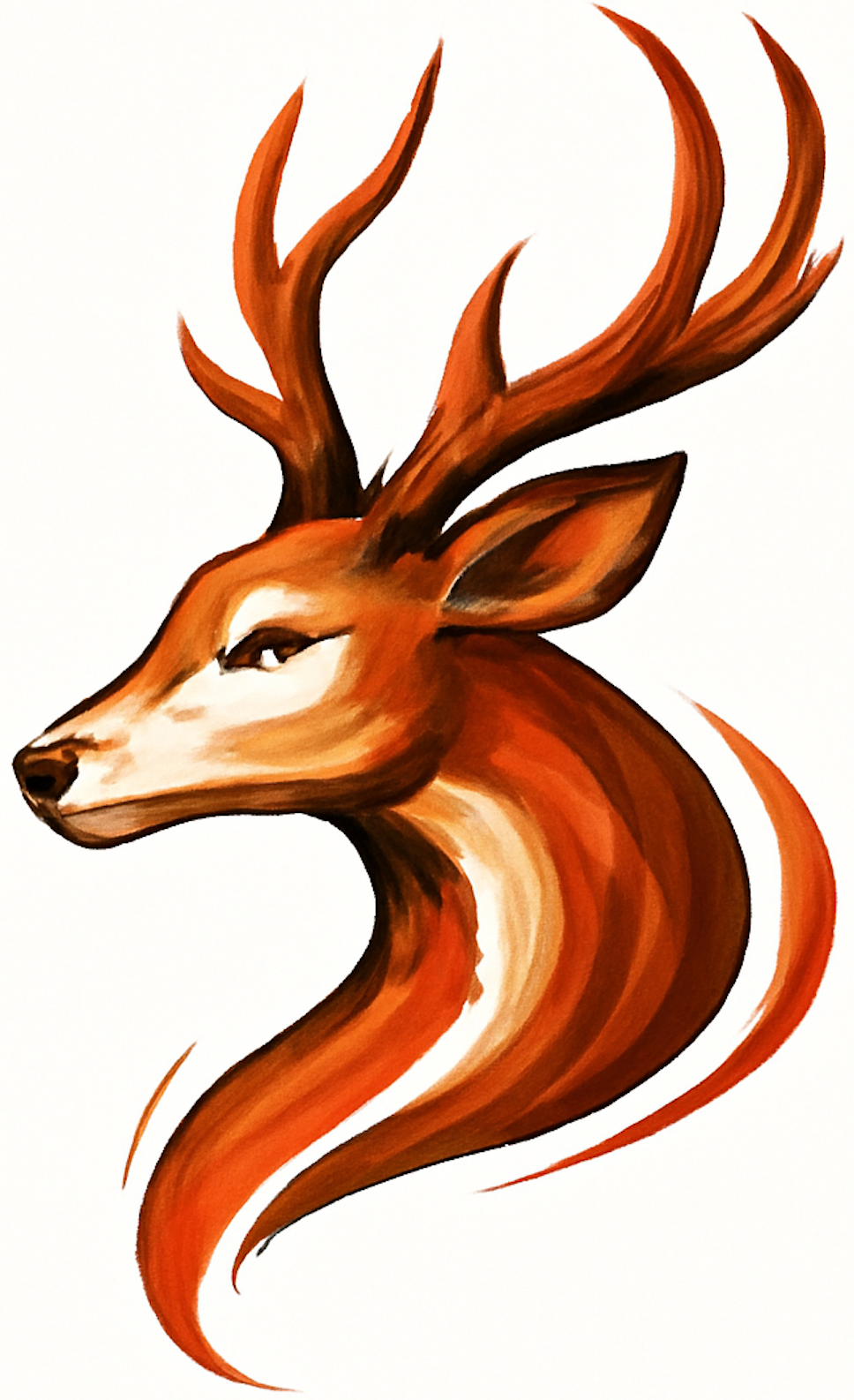}}
  Draft with Diffusion, Verify with Autoregressive Models
}{
  Draft with Diffusion, Verify with Autoregressive Models
}
}

  \icmltitlerunning{Draft with Diffusions, Verify with Autoregressive Models} 

  \icmlsetsymbol{equal}{*}

  \begin{icmlauthorlist}
    \icmlauthor{Zicong Cheng}{thu,seer,sjtu}
    \icmlauthor{Guo-Wei Yang}{seer}
    \icmlauthor{Jia Li}{thu}
    \icmlauthor{Zhijie Deng}{sjtu}
    \icmlauthor{Meng-Hao Guo\textsuperscript{\Letter}}{thu}
    \icmlauthor{Shi-Min Hu}{thu}

  \end{icmlauthorlist}

  \icmlaffiliation{thu}{Tsinghua University}
  \icmlaffiliation{seer}{Proxseer Inc}
  \icmlaffiliation{sjtu}{Shanghai Jiao Tong University}

  \icmlcorrespondingauthor{Meng-Hao Guo}{gmh@tsinghua.edu.cn}

  \icmlkeywords{Machine Learning, ICML}

  \vskip 0.3in
]

\printAffiliationsAndNotice{\ }

\begin{abstract}

Efficiency, as a critical practical challenge for LLM-driven agentic and reasoning systems, is increasingly constrained by the inherent latency of autoregressive (AR) decoding.
Speculative decoding mitigates this cost through a draft–verify scheme, yet existing approaches rely on AR draft models (\aka drafters), which introduce two fundamental issues:  (1) step-wise uncertainty accumulation leads to a progressive collapse of trust between the target model and the drafter,
and 
(2) inherently sequential decoding of AR drafters. 
Together, these factors cause limited speedups.
In this paper, we show that a diffusion large language model (dLLM) drafters can naturally overcome these issues through its fundamentally different probabilistic modeling and efficient parallel decoding strategy.
Building on this insight, we introduce \textbf{\Framework{}}, an efficient speculative decoding framework that drafts with diffusion and verifies with AR models. 
To enable high-quality drafting, \Framework{} employs a two-stage training pipeline to align the dLLM-based drafters with the target AR model, and further adopts single-step decoding to generate long draft segments.
Experiments show \Framework{} reaches draft acceptance lengths of up to 32 tokens, far surpassing the 10 tokens achieved by EAGLE-3. 
Moreover, on HumanEval with Qwen3-30B-A3B, \Framework{} attains a 5.54× speedup, while EAGLE-3 achieves only 2.41×. 
Code, model, demo, etc, will be available at 
\url{https://czc726.github.io/DEER/}

\end{abstract}

\section{Introduction}

\begin{figure}[htbp]
    \centering
    \includegraphics[width=\linewidth]{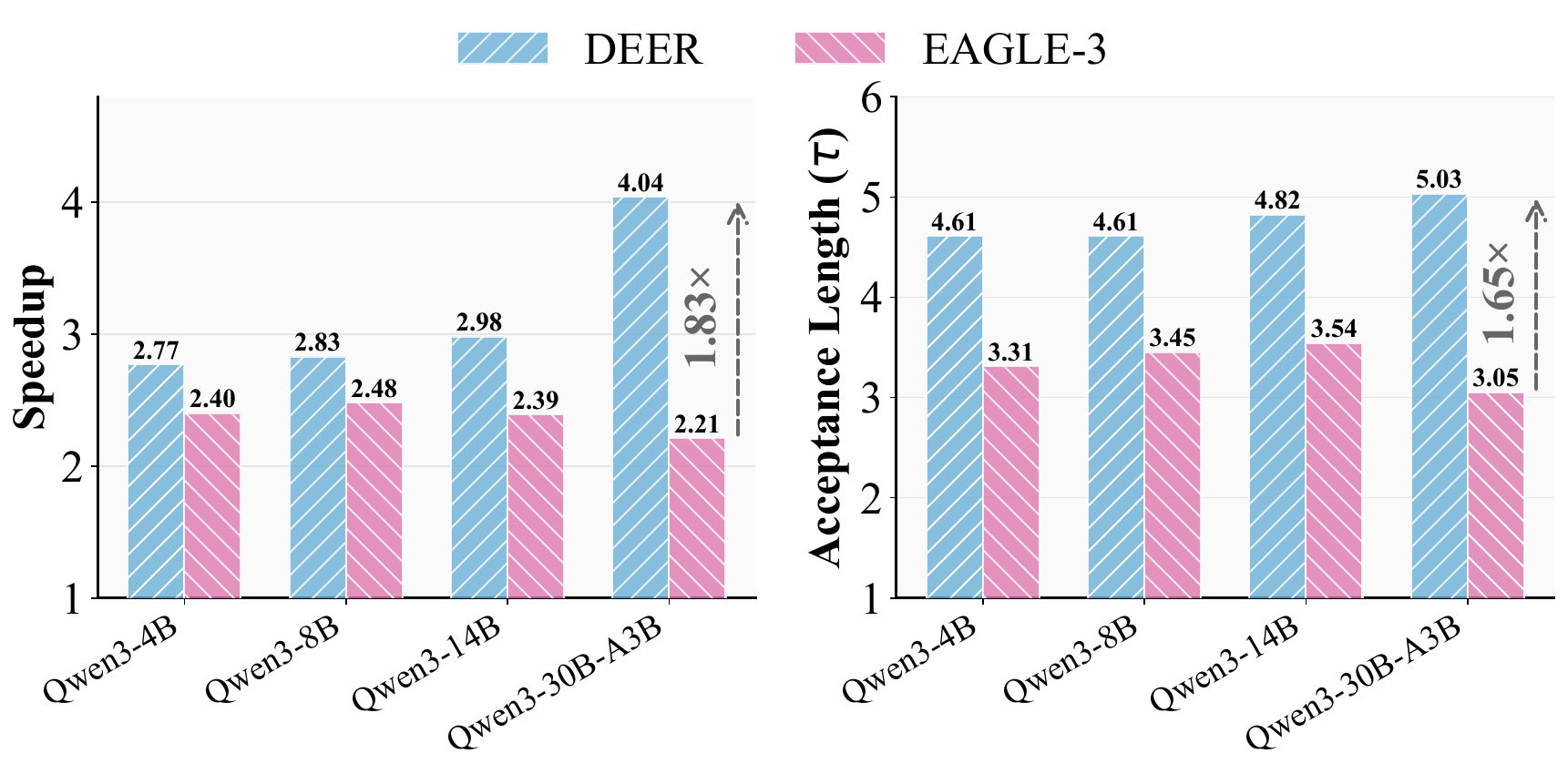}
    \caption{Performance Comparison of \Framework{} and EAGLE-3: Speedup and $\tau$ Across Models (tokens/s) at Temperature=0}
    \label{fig:first_result}
\end{figure}
 
Large language models (LLMs) have fundamentally reshaped the modern AI ecosystem, owing to their remarkable generalization~\cite{MATH25,INDICT,Bee}.
Meanwhile, as the demand for extended context continues to rise, particularly in complex reasoning and agentic tasks, 
efficiency becomes increasingly critical. 
To mitigate this growing bottleneck while preserving model fidelity, 
speculative decoding~\cite{SPS,accs} has  emerged as an effective approach for efficient decoding, providing lossless acceleration by enabling lightweight drafters to propose candidate continuations that are verified by the target model.

However, existing speculative decoding methods overwhelmingly rely on AR drafters, which impose two structural limitations.
On the one hand, 
left-to-right decoding induces step-wise uncertainty accumulation, where uncertainties in early draft tokens propagate through the sequence, progressively degrading alignment with the target model and sharply limiting acceptance length. 
We term this phenomenon as gradual
\textbf{collapse of trust} between the drafter and the target model.
As illustrated in Figure~\ref{fig:motivation}, when the drafter conditions on its own unverified outputs, even a small discrepancy from the target model at early positions is recursively amplified through left-to-right decoding. The draft trajectory gradually drifts outside the acceptance region, causing the verifier to reject increasing portions of the draft. On the other hand, 
AR drafters themselves must decode sequentially, preventing them from exploiting parallel generation.
Together, above structural limitations become the bottleneck that restricts attainable speedups.

\begin{figure}[t]
    \centering
    \includegraphics[width=\linewidth]{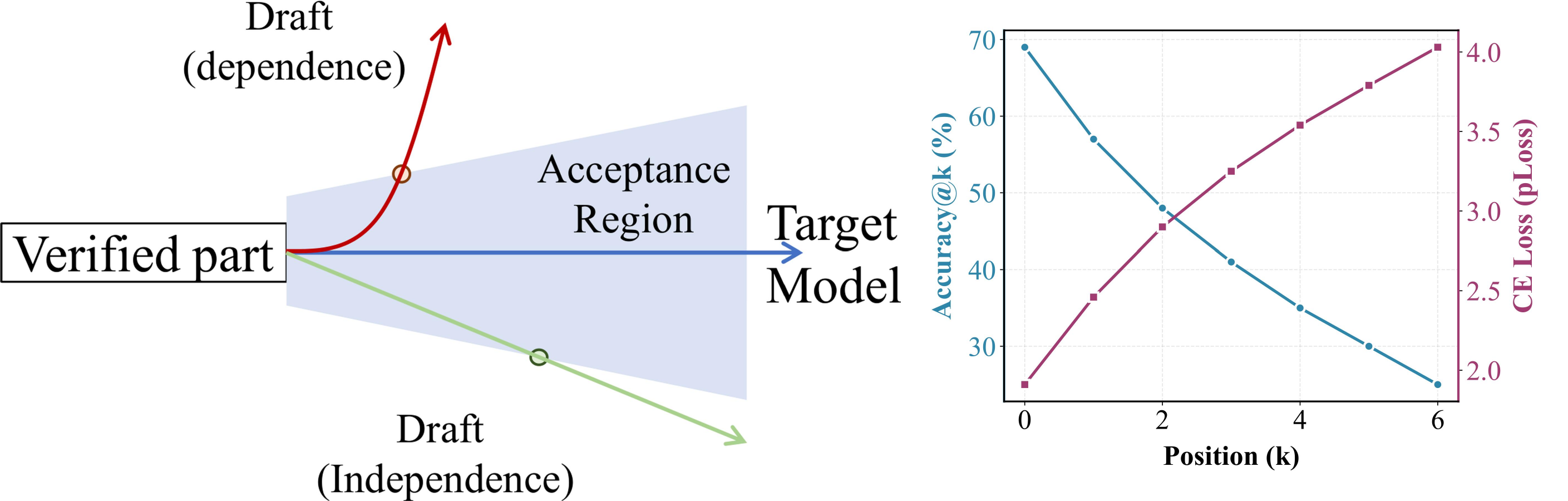}
    \caption{\textbf{Left}: Comparison between the dependence-based and independence-based drafting strategies. T
\textbf{Right}: Accuracy@k and cross-entropy loss across intermediate checkpoints of the EAGLE-3 training pipeline. Together, the plots illustrate draft–backbone alignment and the progression of training performance.}
    \label{fig:motivation}
\end{figure}

We address above two challenges
by introducing \Framework{}, a novel framework that utilizes discrete-space dLLMs as efficient draft generators.
Ideally, unlike AR drafters, a dLLM-based drafter can naturally
generate an entire sequence of tokens in a single denoising process, thereby not only overcoming the efficiency limitations of serial AR decoding, but also theoretically eliminating the multi-step uncertainty accumulation inherent in AR sequential generation.
However, naïvely applying dLLMs as drafters for speculative decoding bring serious distribution mismatch, as standard dLLMs are trained for global sequence generation 
rather than prefix-conditioned local continuation, leading draft proposals are misaligned with the target AR distribution.

To address this challenge, we develop a two-stage training strategy that
adapts pretrained dLLMs into efficient and high-fidelity blockwise drafters.
Stage-1, \textit{AR-Style Continuation Distillation}, aligns the diffusion model with the target AR distribution by training on truncated teacher answers appended with a special SEP marker, enabling stable prefix-conditioned continuation.
Stage-2, \textit{Prefix-Conditioned Accuracy Refinement}, enhances
continuation fidelity through weighted suffix masking with an exponentially decaying loss, improving token-level stability near the AR verification boundary. Combined, these stages produce a dLLM-based drafter that supports reliable blockwise generation.
Furthermore, we observe above training pipeline unlocks an emergent capability we term reliable block regeneration: the ability of the dLLM to repeatedly accept partially masked suffixes and regenerate them in a coherent manner.

During inference, \Framework{} supports one-step block drafting, eliminating left-to-right dependency and enabling significantly longer accepted drafts. 
Across various code-generation benchmarks, \Framework{} achieves acceptance lengths of up to 32 tokens, far exceeding the 10 tokens typical of advanced methods such as EAGLE-3. On HumanEval with Qwen3-30B-A3B, \Framework{} delivers a 5.54$\times$ speedup, surpassing the 2.41$\times$ speedup of EAGLE-3 and establishing dLLM-based drafting as a promising path for acceleration.

In summary, our key contributions are as follows:

\begin{itemize} 
\setlength{\itemsep}{0pt} 

\item 
 We introduce \Framework{}, the first speculative decoding framework that relies \emph{exclusively} on a discrete-space dLLM as the drafter, removing the need for auxiliary AR models or hybrid drafting architectures. 
Further, we reveal a novel generative capability in \Framework{}, termed reliable block regeneration, wherein the dLLM can perform genuinely blockwise generation from incrementally masked suffixes.

\item 
 We propose a two-stage alignment method that adapts dLLMs to the structural requirements of speculative decoding: Stage-1 resolves the distribution mismatch for prefix-conditioned continuation, while Stage-2 provides fine-grained local accuracy through exponentially weighted suffix masking.

\item  
 Experiments on diverse benchmarks and model scales (Qwen3--4B to 30B), \Framework{} consistently outperforms existing approaches. For instance, on HumanEval with Qwen3-30B-A3B, \Framework{} attains a 5.54× speedup, while EAGLE-3 achieves only 2.41×.

\end{itemize}

\section{Related Work}

\subsection{Speculative Decoding}

\textbf{Autoregressive and Tree-based Drafting.} Standard methods employ small auxiliary AR models~\cite{SPS} or manipulate hidden states to predict tree-structured continuations (e.g., Medusa~\cite{Medusa}, Hydra~\cite{Hydra}, EAGLE series~\cite{Eagle,Eagle-2,Eagle-3}). While effective, these methods remain inherently sequential: draft tokens are generated left-to-right, meaning early uncertainties propagate and corrupt the draft chain. \textbf{In contrast}, \Framework{} eliminates this serial dependency by generating a full block of tokens in a single diffusion step, preventing uncertainty accumulation even at long draft lengths.

\textbf{$N$-gram and Heuristic Drafting.}
Lightweight approaches like Lookahead~\cite{Lookahead} and DiffuSpec~\cite{DiffuSpec} construct drafts via $n$-gram matching or retrieval.
While efficient, they lack global context, causing acceptance rates to drop on complex sequences. \Framework{} leverages dLLMs’ global modeling to keep drafts coherent and contextually consistent over long spans.

\textbf{Diffusion-based Drafting.} Speculative Diffusion Decoding (SDD)~\cite{SDD} pioneered using dLLMs for drafting but relies on continuous-space, multi-step denoising. This prohibits precise temperature control and introduces step-wise drift from the AR verifier. \Framework{} diverges by operating in discrete space with a strictly aligned one-step generation process, ensuring high compatibility with the target AR distribution.

\subsection{Self-Drafting and Hybrid Architectures}

\textbf{Latent and Intermediate Self-Drafting.} Methods such as SSDD~\cite{SSDD} and SSMD~\cite{SSMD} exploit intermediate noisy states or self-generated diffusion logits to form drafts. However, these intermediate representations are often noisy and unaligned with the final autoregressive objective. \Framework{} avoids this by utilizing a fully denoised, alignment-tuned output distribution, yielding significantly cleaner and more stable proposals.

\textbf{Hybrid AR-Diffusion Training.} Approaches like TiDAR~\cite{TiDAR} retrain AR models to jointly perform diffusion-style generation. While this unifies drafting and verification, the dual-objective training is computationally expensive and often induces conflict that degrades base model performance. \Framework{} employs a modular design with a dedicated, lightweight dLLM. This avoids expensive retraining of the target LLM and preserves its original capabilities without objective conflicts.

\subsection{Diffusion language models}

Early work such as D3PM~\cite{D3PM} and Diffusion-LM~\cite{Diffusion-lm} introduced diffusion language models in continuous spaces, while later approaches like LLADA~\cite{LLADA} and Dream~\cite{Dream} scaled them to discrete tokens and larger model sizes. A key benefit of diffusion models is their ability to generate multiple tokens in parallel, reducing autoregressive dependency. Leveraging this, dLLMs support coherent block-wise token generation for efficient speculative decoding.

\begin{figure*}[t]
    \centering
    \includegraphics[width=0.9\linewidth]{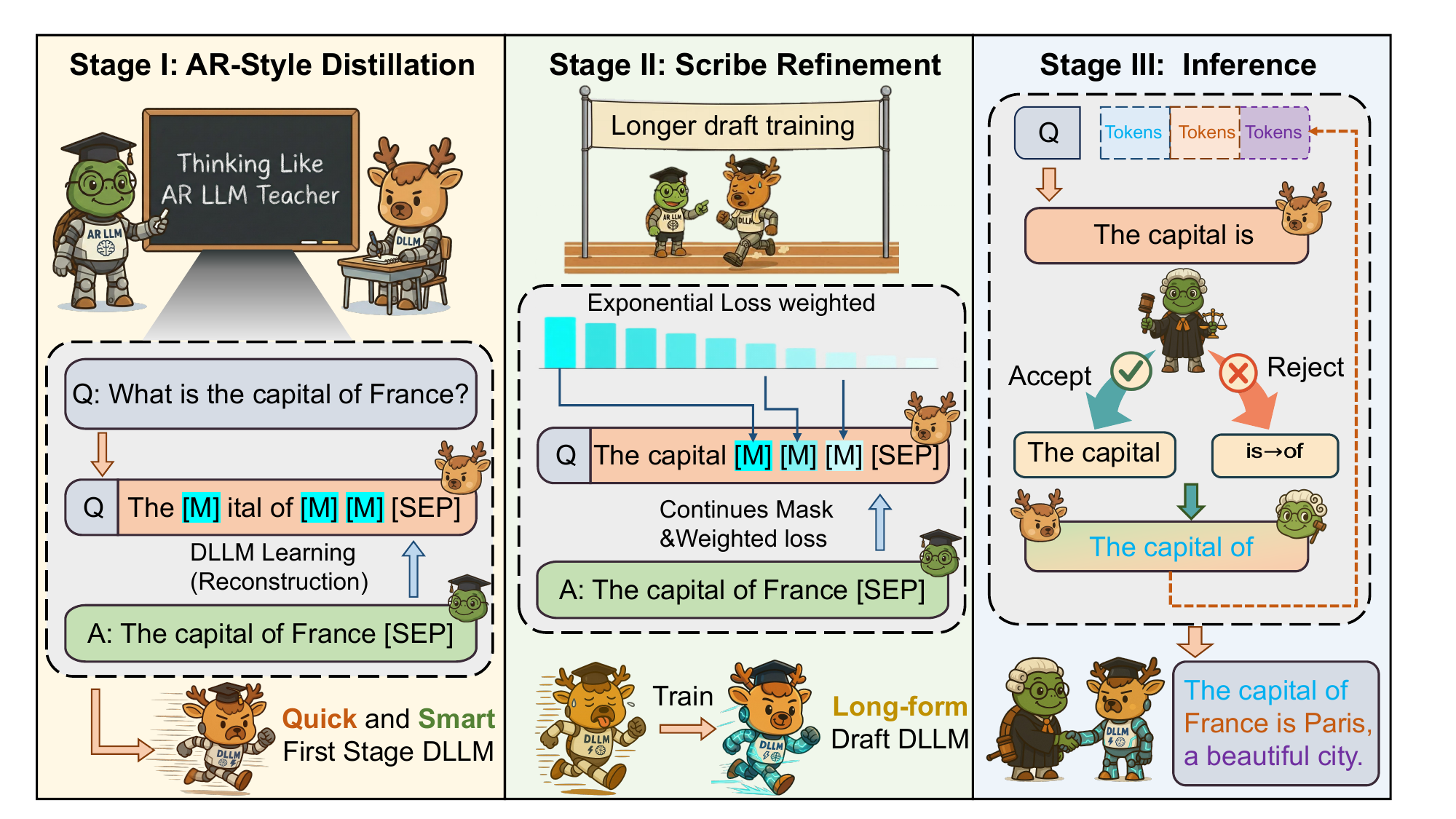}
    \caption{Overview of the \Framework{} pipeline.
\textbf{Stage~\textit{I} (AR-Style Continuation Distrilling)} structurally adapts the dLLM to generate full suffix blocks from prefix + [SEP] using truncated teacher answers.
\textbf{Stage~\textit{II} (Scribe Refinement)} stabilizes local coherence through weighted suffix masking with an exponential decaying loss.
\textbf{Stage~\textit{III} Inference} performs draft-then-verify decoding, where large draft blocks proposed by the dLLM are accepted or corrected by the target AR model, accelerating inference while preserving quality.}
    \label{fig:pipeline}
\end{figure*}
\section{\Framework{}}
Speculative decoding accelerates autoregressive (AR) inference by enabling a lightweight model to propose multiple tokens in parallel, which are then verified by the target AR model. However, existing AR-based drafters inevitably suffer from \textbf{left-to-right uncertainty accumulation} (\aka{gradual collapse of trust between drafters and target model}): each drafted token conditions on previously unverified ones, amplifying early deviations. As draft depth increases, acceptance sharply declines, limiting speedup.

DEER resolves this bottleneck by leveraging a dLLM as the drafter. Unlike AR generation, diffusion models jointly reconstruct the entire suffix, making proposal quality largely invariant to token depth. To adapt pretrained dLLMs to prefix-conditioned continuation, we propose a two-stage \textbf{Diffusion-to-Autoregressive (D2A) Alignment} pipeline, followed by a lightweight block-wise verification procedure.

\subsection{Diffusion-to-AR Alignment}

Diffusion models operate through global denoising and are not inherently 
consistent with AR-style prefix continuation. Directly employing a pretrained 
dLLM as a drafter causes severe distribution mismatch, producing 
unstable suffix predictions.
To make dLLMs suitable for as drafters,
D2A adapts them to 
prefix-conditioned continuation behavior.

\paragraph{Notation}
We denote: $p_{\mathrm{AR}}$: AR teacher model;
$x_0$: original tokens with masked future spans;
$x_t$: noised tokens at timestep $t$;
$L$: full length of a teacher-generated answer;
$l_q$: prefix/question length;
$\mathbf{M}$: mask token;
\texttt{SEP}: separator token indicating truncation;
$p_\theta$: aligned dLLM.

\begin{table*}[t]
\centering
\renewcommand{\tabcolsep}{3.0mm}

\begin{threeparttable}
\caption{Performance Comparison of Acceleration Methods Across Models  (temperature=0.6), with KV cache}
\label{tab:model_performance_comparison_tem06}
\small

\begin{tabular}{l *{5}{c c}}

\toprule
\multicolumn{1}{c}{\textbf{Method}} &
\multicolumn{2}{c}{\textbf{MBPP}} &
\multicolumn{2}{c}{\textbf{CodeAlpacaPy}} &
\multicolumn{2}{c}{\textbf{HumanEval}} &
\multicolumn{2}{c}{\textbf{LiveCodeBench}} &
\multicolumn{2}{c}{\textbf{Mean}}\\
\cmidrule(lr){2-3} \cmidrule(lr){4-5} \cmidrule(lr){6-7} \cmidrule(l){8-9} \cmidrule(l){10-11}

 &
Speedup & $\tau$ &
Speedup & $\tau$ &
Speedup & $\tau$ &
Speedup & $\tau$ &
Speedup & $\tau$ \\
\midrule

\rowcolor{gray!15}
\multicolumn{11}{c}{\textbf{Qwen3-4B}} \\
EAGLE3 & ×1.98 & 2.60 & ×1.86 & 2.43 & ×2.01 & 2.63 & ×1.79 & 2.51 & ×1.91
& 2.54 \\
\Framework{} & ×2.82 & 4.79 & ×2.51 & 4.10 & ×2.59 & 4.51 & ×2.14 & 4.25 & ×2.52
& 4.41 \\

\rowcolor{gray!15}
\multicolumn{11}{c}{\textbf{Qwen3-8B}} \\
EAGLE3 & ×2.12 & 3.46 & ×1.55 & 2.58 & ×2.18 & 3.23 & ×1.33 & 2.42 & ×1.80 & 2.92 \\
\Framework{} & ×3.35 & 4.84 & ×2.40 & 3.81 & ×3.18 & 4.66 & ×1.59 & 2.93 & ×2.63 & 4.06 \\

\rowcolor{gray!15}
\multicolumn{11}{c}{\textbf{Qwen2-7B}} \\
EAGLE3 & ×2.40 & 3.28 & ×2.18 & 3.23 & ×2.41 & 3.25 & ×2.39 & 2.91 & ×2.35 & 3.18 \\
\Framework{} & ×2.45 & 3.81 & ×2.27 & 3.50 & ×2.52 & 3.90 & ×2.84 & 4.39 & ×2.52 & 3.90 \\

\rowcolor{gray!15}
\multicolumn{11}{c}{\textbf{Qwen3-14B}} \\
EAGLE3 & ×1.89 & 2.22 & ×2.01 & 2.44 & ×1.91 & 2.61 & ×1.75 & 2.30 & ×1.89& 2.39 \\
\Framework{} & ×3.67 & 4.93 & ×2.80 & 3.69 & ×3.50 & 5.00 & ×2.53 & 3.93 & ×3.13
& 4.39 \\

\rowcolor{gray!15}
\multicolumn{11}{c}{\textbf{Qwen3-30B-A3B}} \\
EAGLE3 & ×2.07 & 2.35 & ×1.91 & 2.31 & ×2.14 & 2.57 & ×1.90 & 2.38 & ×2.01 & 2.40 \\
\Framework{} & ×3.69 & 4.79 & ×3.23 & 3.82 & ×4.32 & 5.48 & ×3.24 & 4.09 & ×3.62
& 4.45 \\
\bottomrule
\end{tabular}
\end{threeparttable}
\end{table*}

\begin{table*}[t]
\centering
\renewcommand{\tabcolsep}{3.0mm}
\begin{threeparttable}
\caption{Performance Comparison of Acceleration Methods Across Models (tokens/s), temperature=0, with KV cache}
\label{tab:model_performance_comparison_tem0}
\small

\renewcommand{\arraystretch}{1.15}

\begin{tabular}{l *{5}{cc}}
\toprule
\textbf{Method} &
\multicolumn{2}{c}{\textbf{MBPP}} &
\multicolumn{2}{c}{\textbf{CodeAlpacaPy}} &
\multicolumn{2}{c}{\textbf{HumanEval}} &
\multicolumn{2}{c}{\textbf{LiveCodeBench}} &
\multicolumn{2}{c}{\textbf{Mean}} \\
\cmidrule(lr){2-3}
\cmidrule(lr){4-5}
\cmidrule(lr){6-7}
\cmidrule(lr){8-9}
\cmidrule(lr){10-11}
& Speedup & $\tau$ &
  Speedup & $\tau$ &
  Speedup & $\tau$ &
  Speedup & $\tau$ & 
  Speedup & $\tau$ 
  \\
\midrule

\rowcolor{gray!20}
\multicolumn{11}{c}{\textbf{Qwen3-4B}} \\
MEDUSA & ×1.20 & 2.07 & ×1.22 & 1.85 & ×1.34 & 1.94 & ×1.17 & 1.89 & ×1.23 & 1.94\\
Hydra  & ×2.33 & 2.70 & ×2.17 & 2.50 & ×2.30 & 2.53 & ×2.08 & 2.54 & ×2.22 & 2.58\\
EAGLE3 & ×2.59 & 3.37 & ×2.20 & 3.11 & ×2.47 & 3.26 & ×2.33 & 3.50 & ×2.40 & 3.31 \\
\Framework{}  & ×2.86 & 4.91 & ×2.58 & 4.21 & ×2.97 & 4.68 & ×2.67 & 4.63 & ×2.77 & 4.61 \\

\rowcolor{gray!20}
\multicolumn{11}{c}{\textbf{Qwen3-8B}} \\
MEDUSA & ×1.29 & 1.93 & ×1.33 & 1.98 & ×1.32 & 1.97 & ×1.34 & 1.99 & ×1.32 &1.97 \\
Hydra  & ×2.37 & 2.79 & ×2.02 & 2.53 & ×2.39 & 2.58 & ×2.23 & 2.72 & ×2.25 &2.66 \\
EAGLE3 & ×2.59 & 3.31 & ×2.21 & 3.25 & ×2.65 & 3.87 & ×2.46 & 3.38 & ×2.48 &3.45 \\
\Framework{}  & ×3.00 & 5.12 & ×2.35 & 4.06 & ×3.30 & 5.00 & ×2.67 & 4.27 & ×2.83 &4.61 \\

\rowcolor{gray!20}
\multicolumn{11}{c}{\textbf{Qwen2-7B}} \\
EAGLE3 & ×2.36 & 3.32 & ×2.40 & 3.27 & ×2.65 & 3.34 & ×2.29 & 2.94 & ×2.43 &3.22\\
\Framework{}  & ×2.45 & 3.57 & ×2.50 & 3.69 & ×2.79 & 4.21 & ×3.06 & 4.96 & ×2.70 &4.11 \\

\rowcolor{gray!20}
\multicolumn{11}{c}{\textbf{Qwen3-14B}} \\
EAGLE3 & ×2.50 & 3.52 & ×2.13 & 3.45 & ×2.62 & 3.72 & ×2.30 & 3.48 & ×2.39 &3.54 \\
\Framework{}  & ×3.18 & 5.28 & ×2.43 & 3.98 & ×3.59 & 5.72 & ×2.73 & 4.31 & ×2.98 & 4.82 \\

\rowcolor{gray!20}
\multicolumn{11}{c}{\textbf{Qwen3-30B-A3B}} \\
EAGLE3 & ×2.22 & 3.07 & ×2.06 & 2.89 & ×2.41 & 3.21 & ×2.14 & 3.01 & ×2.21 &3.05 \\
\Framework{}  & ×4.00 & 4.87 & ×3.08 & 4.04 & ×5.54 & 6.58 & ×3.52 & 4.62 & ×4.04 &5.03 \\
\bottomrule
\end{tabular}
\end{threeparttable}
\end{table*}

\begin{table}[t]
\centering
\caption{Average accepted-token lengths on Qwen3-30B-A3B with and without Stage~II refinement.}
\label{tab:quench_refine}
\begin{tabular}{lcc}
\toprule
\textbf{Benchmark} & \textbf{w/o Refinement} & \textbf{w/ Refinement} \\
\midrule
MBPP            & 4.74 & 4.87 \\
CodeAlpacaPy    & 3.47 & 4.04 \\
HumanEval       & 5.38 & 6.58 \\
LiveCodeBench   & 3.87 & 5.03 \\
\bottomrule
\end{tabular}
\end{table}

\newcounter{stage}
\setcounter{stage}{1}
\subsubsection{Stage~\textit{\Roman{stage}}: AR-style Distillation}

A dLLM pre-trained with full-sentence denoising does not inherently model
causal continuation: if the prefix is truncated, its denoising process still
implicitly relies on future tokens that are no longer available. To enable
prefix-conditioned generation, we finetune the model to imitate an
 AR teacher on continuation-style data.

Given a teacher-generated answer
$
\mathcal{A} = \{ a_n^{1:l_n} \},
$
we randomly truncate each answer, mask the suffix, and append a
\texttt{SEP} token to mark the continuation boundary. The dLLM observes a
noisy version $x_t$ and is trained to denoise only the masked continuation:

\begin{equation}
\label{eq:d2a-stage1}
\mathcal{L}_{\mathrm{Distill}}
=
-\mathbb{E}_{t, x_0, x_t}
\left[
\frac{1}{t}
\sum_{i=l_q}^{L-1}
\mathbf{1}[x_t^i=\mathbf{M}]\,
r_i
\right].
\end{equation}

\begin{equation}
\label{eq:d2a-ri}
r_i = \log p_\theta(x_0^i \mid x_t)
\end{equation}

This first stage adapts the dLLM to a setting where the past is observed but
the future must be predicted, aligning its continuation behavior with the AR
teacher and making it compatible with speculative decoding.

\setcounter{stage}{2}
\subsubsection{Stage~\textit{\Roman{stage}}: Scribe Refinement}

While the above training enables causal continuation, speculative acceptance is
particularly sensitive to the tokens that appear immediately after the prefix.
To refine accuracy precisely in this region, we mask only the last
$
R \sim \mathrm{Uniform}(1,96)
$
tokens of the answer, instead of the entire suffix. Tokens closer to the
prefix are emphasized with exponentially increasing weights:

\begin{equation}
\label{eq:d2a-weights}
w_i = \alpha^{R - i}, \qquad i = 1, \dots, R,
\end{equation}

leading to the refinement objective:
\begin{equation}
\label{eq:d2a-stage2}
\mathcal{L}_{\mathrm{Refine}}
= -\mathbb{E}_{t, x_0, x_t} \left[
\frac{1}{t}
\sum_{i=l_q}^{L-1}
w_i\,
\mathbf{1}[x_t^{i}=\mathbf{M}]\,
r_i
\right].
\end{equation}

Through this masked-span curriculum, the dLLM increasingly concentrates
capacity on the region where speculative verification first interacts with the
draft, leading to more reliable block acceptance.

\subsection{Inference}
\label{sec:inference}

At inference time, we use the aligned dLLM as a parallel drafter and an
 AR model as an exact verifier within a block-wise speculative
decoding scheme. Given a current prefix $\mathbf{x}_{1:j}$, the dLLM proposes a
block of $k$ tokens in parallel:
$
\hat{\mathbf{y}}_{j+1:j+k}
\sim
q_\theta(\cdot \mid \mathbf{x}_{1:j}),
$
and the AR model then decides, token by token, whether to accept each proposal.

For the $i$-th token in the block, we compute an acceptance probability
\begin{equation}
\label{eq:accept}
\alpha_i
=
\min\!\left(
1,\;
\frac{
p_{\mathrm{AR}}(\hat{y}_{j+i} \mid \mathbf{x}_{1:j+i-1})
}{
q_\theta(\hat{y}_{j+i} \mid \mathbf{x}_{1:j})
}
\right),
\quad i = 1,\dots,k.
\end{equation}
With probability $\alpha_i$ the draft token is accepted; otherwise it is
replaced by an AR sample:
\begin{equation}
\label{eq:resample}
\hat{y}_{j+i}
\;\propto\;
\max\Bigl(0,\, p_{\mathrm{AR}}(\cdot \mid \mathbf{x}_{1:j+i-1}) - q_\theta(\cdot \mid \mathbf{x}_{1:j})\Bigr),
\end{equation}
In both cases, the chosen token (either the accepted draft or the AR resample)
is appended to the prefix and becomes part of the context for subsequent
positions.

\paragraph{Uncertainty accumulation vs.\ stable block proposals.}
In classical speculative decoding with an autoregressive drafter, the draft
distribution at position $i$ depends on previously sampled draft tokens:
$
q^{\mathrm{AR}}(\hat{y}_i \mid \mathbf{x}_{1:j}, \hat{\mathbf{y}}_{1:i-1})
\neq
q^{\mathrm{AR}}(\hat{y}_i \mid \mathbf{x}_{1:j}),
$
so any divergence between the drafter and $p_{\mathrm{AR}}$ at early positions
propagates to later ones. As a result, the distribution mismatch
$
\mathrm{KL}\!\left(
p_{\mathrm{AR}}(\hat{y}_i \mid \mathbf{x}_{1:j+i-1})
\;\big\|\;
q^{\mathrm{AR}}(\hat{y}_i \mid \mathbf{x}_{1:j}, \hat{\mathbf{y}}_{1:i-1})
\right)
$
tends to grow with $i$, leading to left-to-right uncertainty accumulation and
rapidly decreasing acceptance rates.

In contrast, the dLLM uses block-wise masked conditioning, which yields
$
q_\theta(\hat{y}_i \mid \mathbf{x}_{1:j}, \hat{\mathbf{y}}_{1:i-1})
=
q_\theta(\hat{y}_i \mid \mathbf{x}_{1:j}),
$
so the proposal at position $i$ is independent of previously drafted tokens.
The mismatch between $p_{\mathrm{AR}}$ and $q_\theta$ at depth $i$ is therefore
determined solely by how well $q_\theta(\cdot \mid \mathbf{x}_{1:j})$ matches
$p_{\mathrm{AR}}(\cdot \mid \mathbf{x}_{1:j+i-1})$, rather than by accumulated
errors in earlier drafts. In other words, the drafter does not \emph{amplify}
its own past mistakes, which is the primary source of degradation in
AR-based speculative decoding.

\paragraph{Overall decoding procedure.}
Putting these pieces together, decoding proceeds by alternating between (i)
parallel block proposals from the dLLM and (ii) token-wise validation and
correction by the AR model. Most computation is shifted to the parallelizable
drafting step, while the AR model is used only for lightweight verification to
guarantee exact marginal correctness. The full inference routine is summarized
in Algorithm~\ref{alg:framework}.

\begin{algorithm}[t]
\caption{\Framework{} Inference}
\label{alg:framework}
\begin{algorithmic}[1]
\STATE \textbf{Input:} prefix $\mathbf{x}_{1:j}$, block size $k$
\WHILE{not EOS and length limit not reached}
    \STATE Sample draft block 
    $
    \hat{\mathbf{y}}_{j+1:j+k}
    \sim q_\theta(\cdot \mid \mathbf{x}_{1:j})
    $
    \FOR{$i = 1 \dots k$}
        \STATE Compute $\alpha_i$ via Eq.~\ref{eq:accept}
        \STATE Sample $u \sim \mathrm{Uniform}(0,1)$
        \IF{$u \le \alpha_i$}
            \STATE accept $\hat{y}_{j+i}$ and set
            $\mathbf{x}_{1:j+i} \gets \mathbf{x}_{1:j+i-1} \circ \hat{y}_{j+i}$
        \ELSE
            \STATE sample 
            $
            \hat{y}_{j+i} \sim p_{\mathrm{AR}}(\cdot \mid \mathbf{x}_{1:j+i-1})
            $
            via Eq.~\ref{eq:resample} and set
            $\mathbf{x}_{1:j+i} \gets \mathbf{x}_{1:j+i-1} \circ \hat{y}_{j+i}$
        \ENDIF
        \IF{$\hat{y}_{j+i}$ is EOS}
            \STATE \textbf{break}
        \ENDIF
    \ENDFOR
    \STATE Update $j \gets \text{length}(\mathbf{x})$
\ENDWHILE
\end{algorithmic}
\end{algorithm}

\section{Experiment}
In this section, we conduct extensive experiments to evaluate the effectiveness and generalizability of \Framework{}. Our empirical study is organized around the following research questions: \vspace{-10pt} 
\begin{itemize} 
\setlength{\itemsep}{0pt} 
\item \textbf{RQ1:} How does \Framework{} perform on code generation tasks in terms of inference efficiency and draft acceptance distribution? 
\item \textbf{RQ2:} What is the impact of the Stage~\textit{\Roman{stage}} and how sensitive is it to hyperparameter choices? 
\item \textbf{RQ3:} How does \Framework{} support batch inference scalability? 
\item \textbf{RQ4:} Does \Framework{} endow dLLMs with new generative capabilities beyond standard denoising? 
\item \textbf{RQ5:} How does \Framework{} perform on mathematical reasoning benchmarks? 
\end{itemize}
\subsection{Experimental Settings}

\textbf{Datasets}.
For code generation, we train our draft model on the OpenCodeInstruct dataset~\cite{opencodeinstruct} and evaluate on HumanEval~\cite{humaneval}, MBPP~\cite{MBPP}, LiveCodeBench~\cite{lcb}, and the Python subset of CodeAlpaca~\cite{CodeAlpaca}.For the math reasoning setting, we train using data from UltraChat ~\cite{ultrachat} and ShareGPT~\cite{ShareGPT}, and evaluate on GSM8K~\cite{gsm8k}, Math500~\cite{Math500}, and Minerva Math ~\cite{Minervamath}.

\textbf{Models}.For code generation experiments, we adopt Open-dLLM \cite{opendLLM} as our base diffusion model and apply our continued-training procedure to obtain the draft model.For math reasoning, we start with Qwen2.5-0.5B-Instruct \cite{qwen2,qwen2.5}, modify it with a diffusion decoding head, and then apply our two-stage continued training to derive the corresponding draft model.

\textbf{Baselines}.We compare \Framework{} with state-of-the-art speculative decoding methods for which official training code is publicly available, including Medusa~\cite{Medusa}, Hydra~\cite{Hydra}, and EAGLE-3~\cite{Eagle-3}.

\textbf{Metrics}. Since \Framework{} preserves the original model weights and employs strict rejection sampling to guarantee lossless speculative decoding, we do not report accuracy-based metrics. Following standard practice in prior work on speculative decoding, we evaluate using two key metrics:
\begin{itemize}
\setlength{\itemsep}{0pt}
\item \textbf{Speedup Ratio}. The empirical end-to-end speedup compared with standard autoregressive decoding.
\item \textbf{Average Acceptance Length ($\tau$)}. The average number of tokens accepted from the draft per drafting–verification cycle, reflecting the effective number of tokens generated in each speculative step.
\end{itemize}
\subsection{Performance on Code Generation (RQ1)}
In this section, DEER is used with KV cache enabled.

\subsubsection{Overall Efficiency}

\begin{figure}[t]
    \centering

    \begin{subfigure}[b]{0.32\linewidth}
        \centering
        \includegraphics[width=\linewidth]{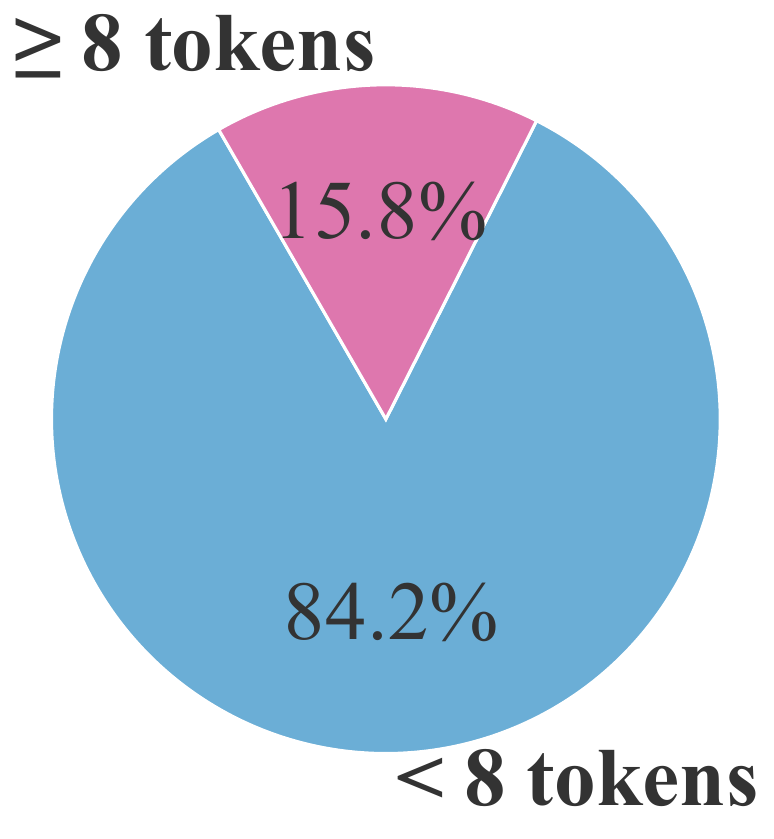}

        \caption{Qwen3-4B}
        \label{fig:pie_4b}
    \end{subfigure}
    \hfill
    \begin{subfigure}[b]{0.32\linewidth}
        \centering
        \includegraphics[width=\linewidth]{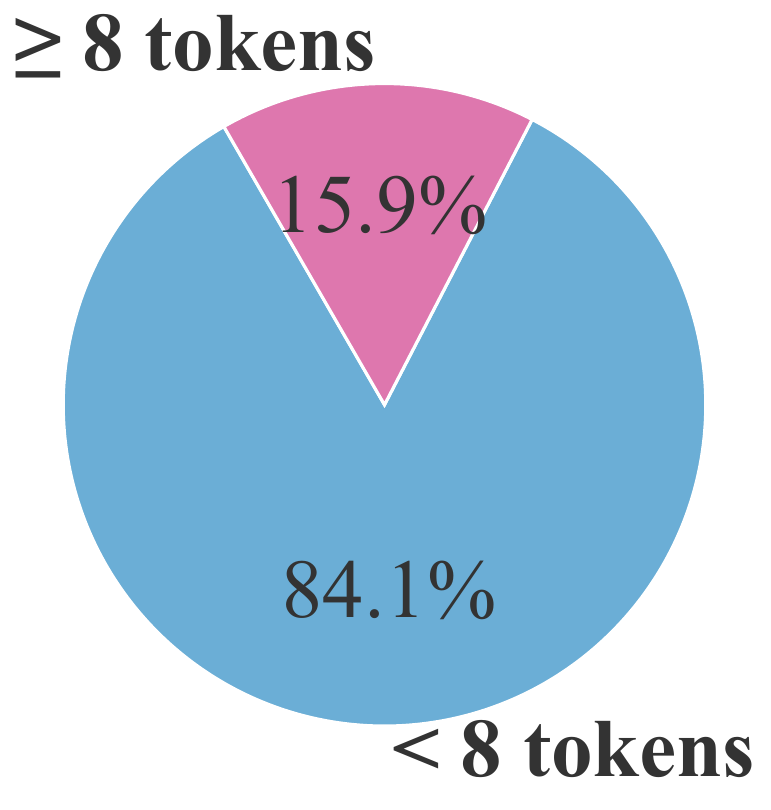}
        \caption{Qwen3-8B.}
        \label{fig:pie_8b}
    \end{subfigure}

    \begin{subfigure}[b]{0.32\linewidth}
        \centering
        \includegraphics[width=\linewidth]{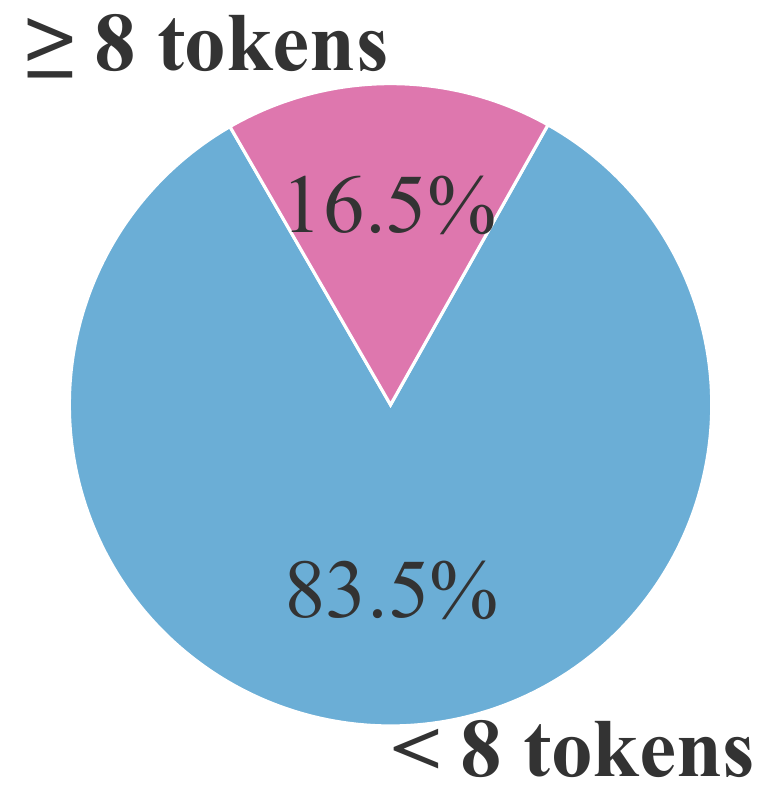}
        \caption{Qwen3-14B.}
        \label{fig:pie_14b}
    \end{subfigure}

    \caption{Proportion of short ($<$8 tokens) and long ($\ge$8 tokens) accepted tokens for different model.} 
    \label{fig:pie_all}
\end{figure}

As shown in Tables~\ref{tab:model_performance_comparison_tem06} and~\ref{tab:model_performance_comparison_tem0}, \Framework{} consistently outperforms state-of-the-art speculative decoding methods across all model scales and datasets, in both average acceptance length $\tau$ and end-to-end speedup.

For Qwen3-30B-A3B at temperature $=0$, \Framework{} achieves an average acceptance length of $5.03$, a $67\%$ increase over EAGLE-3 ($\tau = 3.05$). This demonstrates that speculative decoding can remain highly effective even for modern large models with complex, high-entropy vocabularies. Prior methods are constrained by cumulative left-to-right draft errors, which sharply limit acceptance lengths on such models. In contrast, \Framework{} greatly mitigates this error accumulation through its one-step block-generation mechanism, yielding a substantially more stable verification process.

Across model families, \Framework{} improves $\tau$ by $50\text{--}120\%$ relative to EAGLE-3, depending on model scale and task. On HumanEval, the acceptance length of Qwen3-30B-A3B under \Framework{} is more than twice that of EAGLE-3 ($6.58$ vs.\ $3.21$), and this longer acceptance directly translates into faster decoding: \Framework{} achieves up to $2\times$ the speedup of EAGLE-3 under the same setting. Overall, these results indicate that controlling draft error accumulation is crucial for achieving robust acceleration on contemporary LLMs, and that \Framework{} provides a practical way to do so.

\begin{table}[t]
  \centering
  \renewcommand{\tabcolsep}{5.0mm}
  \caption{Maximum accepted token lengths across models.}
  \label{tab:model_comparison_maxlenth}
  \begin{tabular}{lcc}
    \toprule
    Model & EAGLE-3 & \Framework{} \\
    \midrule
    Qwen3-4B       & 8 & \textbf{32} \\
    Qwen3-8B       & 8 & \textbf{32} \\
    Qwen3-14B      & 8 & \textbf{32} \\
    Qwen3-30B-A3B  & 7 & \textbf{32} \\
    \bottomrule
  \end{tabular}
\end{table}

\begin{table}[t]
  \centering
  \caption{Batch inference performance (tokens/s) on HumanEval across different batch sizes.}
  \label{tab:model_performance_batch}
  \begin{tabular}{lcccc}
    \toprule
    Method & Batch 2 & Batch 4 & Batch 8 & Batch 16 \\
    \midrule
    AR & 34.03 & 32.50 & 38.35 & 49.76 \\
    \Framework{}         & 82.97 & 103.95 & 159.87 & 175.66 \\
    \bottomrule
  \end{tabular}
\end{table}

\subsubsection{Acceptance Distribution Mechanism}
To understand the performance gains, 
we examine how \Framework{} reshapes the distribution of accepted token lengths, which directly reflects the model's ability to propose longer drafts for parallel verification.
As shown in Figure~\ref{fig:pie_all}, the probability of accepting drafts longer than 8 tokens consistently exceeds $15\%$ across all tested models, and this probability slightly increases with model scale,
from $15.8\%$ for Qwen3-4B to $16.6\%$ for Qwen3-30B-A3B. Table~\ref{tab:model_comparison_maxlenth} further shows \Framework{} attains a maximum acceptance length of 32 tokens, whereas EAGLE-3 is limited to $7\text{--}8$ tokens.

These distributions confirm that the one-step generation in \Framework{} effectively decouples token dependencies within each draft block, substantially reducing the left-to-right uncertainty accumulation that typically constrains speculative decoding. As a result, \Framework{} can reliably produce longer contiguous blocks of accepted tokens, leading to higher throughput and greater acceleration potential than conventional autoregressive drafters.

\begin{figure}[t]
\centering
\begin{tcolorbox}[colback=black!3,colframe=black,title=Quicksort Generation (Block Diffusion)]
\small

\textbf{Prompt:}\texttt{Write a Python function for quicksort.}

\vspace{0.25cm}
\textbf{Answer}:\\
\ttfamily
\textcolor{iter0}{def quicksort(arr):}\\
\textcolor{iter0}{\quad if len(arr) <= 1:}\\
\textcolor{iter0}{\quad\quad return arr}\\
\textcolor{iter0}{\quad pivot = arr[0]}\\
\textcolor{iter0}{\quad left = [x for x}
\textcolor{iter1}{in arr if x < pivot]}\\
\textcolor{iter1}{\quad middle = [x for x in arr if x == pivot]}\\
\textcolor{iter1}{\quad right = [x for x}
\textcolor{iter2}{> pivot]}\\
\textcolor{iter2}{\quad return quicksort(left) + middle + quicksort(right)}

\vspace{0.35cm}
\rmfamily

\textit{Color legend:} 
\textcolor{iter0}{Iteration 0} (initial completion), 
\textcolor{iter1}{Iteration 1} (refined extension), 
\textcolor{iter2}{Iteration 2} (final refinement).  
\end{tcolorbox}

\caption{Illustration of block-diffusion generation. Different colors represent 
tokens produced at successive denoising iterations, showing that the dLLM can 
extend partial code blocks without requiring a full-sentence prompt.}
\label{fig:quicksort_blockdiff}
\end{figure}

\begin{figure}[t]
    \centering
    \includegraphics[width=\linewidth]{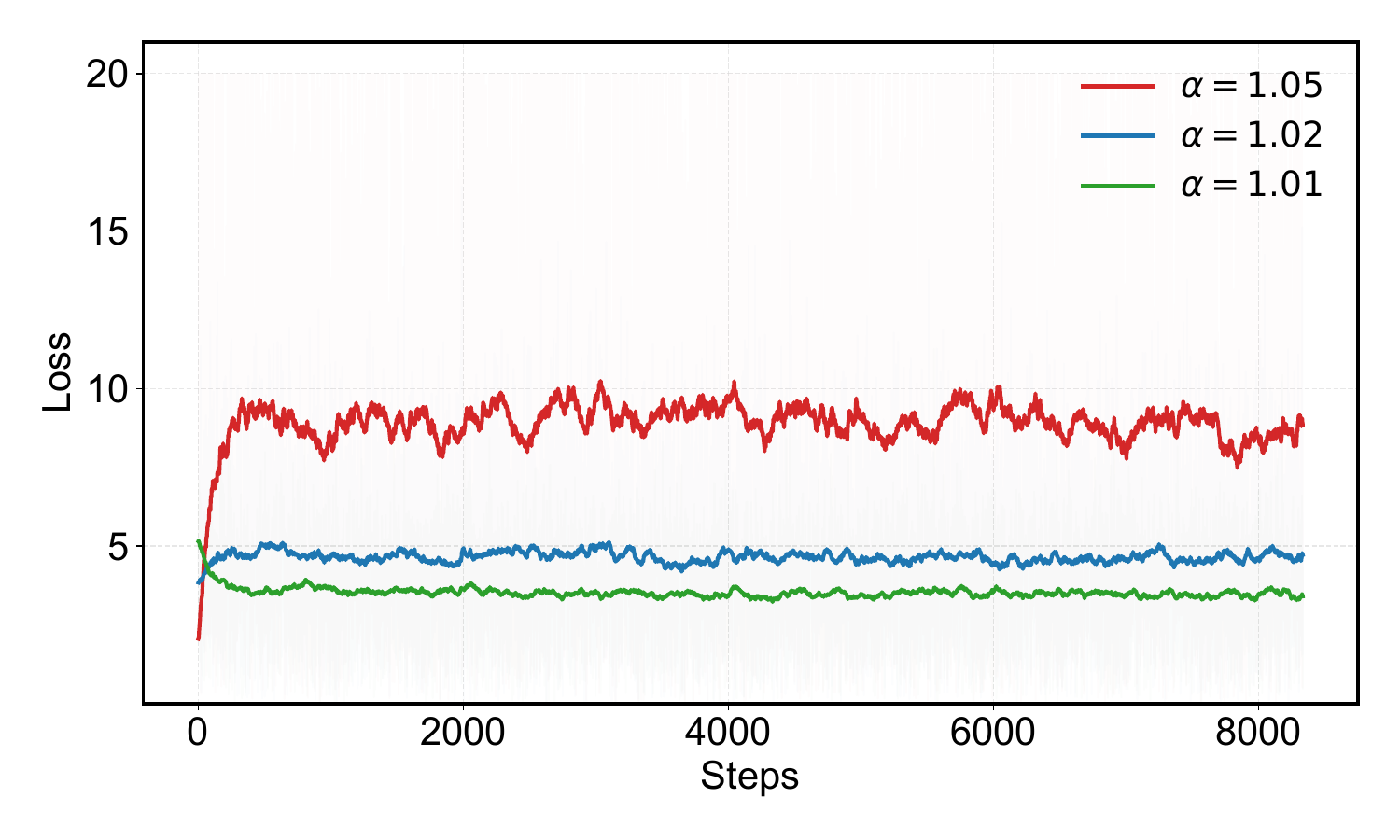}
    \caption{Sensitivity analysis of the exponential weighting factor $\alpha$ during Stage~\textit{II}.The figure illustrates how different values of the exponential weighting coefficient 
    $\alpha\in\{1.01,1.02,1.05\}$ influence the loss trajectory throughout the Quench Refinement process. Each curve is smoothed using an exponential moving average to highlight the underlying optimization dynamics while suppressing stochastic noise.}
    \label{fig:ana_alpha}
\end{figure}

\begin{table}[t]
\centering
\renewcommand{\tabcolsep}{8.0mm}
\caption{Performance of acceleration methods on the \textbf{Qwen3-30B-A3B} model (temperature = 0.6) across math benchmarks. Reported metrics are speedup (×) and Kendall’s $\tau$ correlation.}
\label{tab:model_performance_comparison_math_tem0.6}
\small
\renewcommand{\arraystretch}{1.2}

\begin{tabular}{l|c|c}
\toprule
\textbf{Metric} & \textbf{EAGLE3} & \textbf{\Framework{}} \\
\midrule
\rowcolor{gray!20}
\multicolumn{3}{c}{\textit{Math500}} \\
 Speedup & ×1.89 & ×2.12 \\
 $\tau$   & 2.04  & 2.45  \\
\midrule
\rowcolor{gray!20}
\multicolumn{3}{c}{\textit{GSM8K}} \\
 Speedup & ×1.92 & ×2.23 \\
 $\tau$   & 2.43  & 2.70  \\
\midrule
\rowcolor{gray!20}
\multicolumn{3}{c}{\textit{Minerva Math}} \\
 Speedup & ×1.91 & ×2.02 \\
 $\tau$   & 2.07  & 2.31  \\
\midrule
\rowcolor{gray!20}
\multicolumn{3}{c}{\textit{Mean (across benchmarks)}} \\
 Speedup & ×1.91 & ×2.12 \\
 $\tau$   & 2.18  & 2.47  \\
\bottomrule
\end{tabular}
\end{table}

\subsection{Ablation Study and Sensitivity Analysis (RQ2)}

In this section, we evaluate the role of Stage~\textit{\Roman{stage}} in shaping the dLLM into a 
more reliable draft generator and analyze the sensitivity of its key 
hyperparameters.

\subsubsection{Impact of Stage~\textit{\Roman{stage}}}

We measure the effect of Stage~\textit{\Roman{stage}} by comparing the average accepted-token 
lengths of models trained with and without refinement (Table~\ref{tab:quench_refine}). 
Enabling Stage~\textit{\Roman{stage}} consistently increase the number of accepted tokens across all 
four code-generation benchmarks: from 4.74 to 4.87 on MBPP, 3.47 to 4.04 on 
CodeAlpacaPy, 5.38 to 6.58 on HumanEval, and 3.87 to 5.03 on LiveCodeBench. 

The gap between the two settings grows with benchmark difficulty, with the 
largest reductions on HumanEval and LiveCodeBench (1.20 and 1.16 tokens). 
This suggests that the refinement stage encourages the dLLM to produce suffixes 
that are more tightly aligned with the AR teacher, especially in settings with 
more complex or long-range structure, resulting in more precise and reliable 
drafts.

\subsubsection{Hyperparameter Sensitivity}

Stage~\textit{\Roman{stage}} introduces position-dependent weights parameterized by a scaling 
factor $\alpha$, controling how strongly the loss emphasizes the most recent 
masked tokens. This exponential weighting makes training sensitive to the choice 
of~$\alpha$.

As shown in Figure~\ref{fig:ana_alpha}, when $\alpha = 1.01$, optimization is 
stable and the loss decreases smoothly. Increasing $\alpha$ to $1.02$ yields 
noticeably noisier training curves with slight upward drift, and setting 
$\alpha = 1.05$ leads to early divergence. These results indicate that Stage~\textit{\Roman{stage}} 
has a relatively narrow stability window: overly aggressive weighting amplifies 
gradients near the masked boundary and destabilizes optimization. Within the 
stable regime, however, the refinement stage consistently improves suffix 
alignment without requiring additional data or extended training.

\subsection{Batch Inference Scalability (RQ3)}

We evaluate \Framework{}'s batch inference performance on HumanEval by measuring throughput (tokens/s) under different batch sizes. Since there is currently no mature framework for efficient dLLM+KV-cache deployment, both the baseline autoregressive decoding and \Framework{} are run \emph{without} KV-cache, so the comparison reflects raw model.

Table~\ref{tab:model_performance_batch} reports results for batch sizes 2, 4, 8 and 16. Across all settings, \Framework{} yields substantial speedups over the autoregressive baseline and scales well with batch size. For example, at batch size 8, \Framework{} reaches 159.87 tokens/s, nearly $4\times$ the baseline throughput of 38.35 tokens/s.

These results show that improvements in per-step acceptance length translate into batch-level acceleration: one-step draft generation combined with parallel verification allows \Framework{} to better exploit GPU parallelism, especially at larger batch sizes. Modest efficiency gaps at smaller batch sizes are mainly due to fixed speculative-decoding overheads, which become negligible as batch size increases.

\subsection{Generative Capabilities (RQ4)}

We observe an interesting emergent behavior in dLLMs trained with our \Framework{}: the models are able to perform \emph{reliable block regeneration}. As shown in Figure~\ref{fig:quicksort_blockdiff}, the model can extend an incomplete code segment purely from its local prefix, generating new tokens in a diffusion-style manner without requiring a full-sentence prompt.

Different colors in Figure~\ref{fig:quicksort_blockdiff} indicate tokens produced at different denoising iterations, highlighting how the model incrementally refines and extends the partial block. This illustrates that \Framework{} enables dLLMs to treat block-level continuation as a natural generation mode, even without any architectural modifications such as padding tokens or altered attention patterns.

\subsection{Performance on Mathematical Reasoning (RQ5)}
Since no pretrained dLLMs are publicly available for mathematical-reasoning tasks, we construct our draft model by converting Qwen2.5-0.5B-Instruct into a diffusion model and training it for 40 epochs on the UltraChat dataset~\cite{ultrachat}. This yields only a partially converged dLLM—its standalone generations are not yet semantically reliable—yet \Framework{} still achieves consistent acceleration improvements on mathematical benchmarks.

Table~\ref{tab:model_performance_comparison_math_tem0.6} shows that, despite the weakly trained draft model, \Framework{} surpasses EAGLE-3 across all datasets. On Math500, \Framework{} improves speedup from 1.89× → 2.12× (+12.2\%) and increases the acceptance length from 2.04 → 2.45 (+20.1\%). Similar gains hold for GSM8K (speedup 1.92× → 2.23×, $\tau$ 2.43 → 2.70) and Minerva Math (speedup 1.91× → 2.02×, $\tau$ 2.07 → 2.31). Averaged over the three datasets, \Framework{} delivers a mean speedup of 2.12×, outperforming EAGLE-3’s 1.91×, with an average acceptance length of 2.47, compared to 2.18 for EAGLE-3.

These results demonstrate that \Framework{} generalizes beyond code generation and remains effective even when the underlying dLLM is far from convergence. This suggests that the one-step drafting mechanism produces stable, high-quality proposals for the verifier, enabling reliable speculative decoding in domains where fully trained diffusion-based LLMs are not yet available.

\section{Conclusion}

We presented \textbf{\Framework{}}, a speculative decoding framework that uses a discrete dLLM as the sole drafter, avoiding the left-to-right uncertainty accumulation of autoregressive drafters. To make dLLMs suitable for prefix-conditioned continuation, we introduced a \textit{Diffusion-to-AR Alignment} pipeline that combines AR-style distillation with a lightweight refinement stage near the prefix boundary. 
On multiple code-generation benchmarks and model scales, \Framework{} yields longer accepted blocks and consistent speedups, even without KV caching, demonstrating dLLMs as a practical, highly parallelizable alternative for efficient LLM decoding.

\bibliography{example_paper}
\bibliographystyle{icml2026}

\appendix
\onecolumn

\section{Details of Experimental Settings}
\label{app:exp_details}

All experiments are conducted on a cluster equipped with eight NVIDIA A100 GPUs with 80 GB of memory each.

\paragraph{Draft model for code tasks.}
For code generation benchmarks, we use a 0.5B-parameter diffusion drafter obtained by modifying the open-dCoder checkpoint into a discrete diffusion language model. 
In Stage~\textit{I} (AR-style continuation distillation), we optimize the drafter with the AdamW optimizer, using a learning rate of $1\times10^{-4}$ and training for 1 epoch over the code training corpus. 
In Stage~\textit{II} (prefix-conditioned refinement), we continue training the same drafter with AdamW, a learning rate of $5\times10^{-5}$, and 1 epoch on a subset of 100k examples.

\paragraph{Draft model for math tasks.}
For mathematical reasoning benchmarks, we also use a 0.5B-parameter drafter, initialized from Qwen2.5-0.5B. We convert the original autoregressive checkpoint into a diffusion language model and train it on the UltraChat dataset~\cite{ultrachat} for 40 epochs.
In Stage~\textit{I}, we apply AR-style continuation distillation with AdamW, a learning rate of $1\times10^{-4}$, and train for 5 epochs. 
In Stage~\textit{II}, we perform the refinement stage with AdamW, using a learning rate of $1\times10^{-4}$ for 1 additional epoch.

\section{More Sensitivity Analysis and Experiment}

\subsection{Details of Accept lenth}

As shown in Figure~\ref{fig:combined}, we observe an intriguing pattern in the empirical acceptance-length distribution. For accepted lengths below 30, DEER exhibits an approximately exponential decay, which is consistent with the behavior reported for most speculative decoding methods. However, once the acceptance length exceeds 30 and approaches the maximum range, the probability mass starts to increase again, and this resurgence is quite pronounced. We refer to this phenomenon as the \emph{long-block resurgence effect}. We argue that this effect provides further evidence for our motivation regarding uncertainty accumulation: when later draft tokens are no longer conditioned on earlier draft tokens from the drafter, they are less exposed to left-to-right error propagation and can therefore support substantially longer accepted drafts.

\begin{figure}[h!] 
    
\begin{subfigure}{0.32\textwidth}
    \centering
    \includegraphics[width=\linewidth]{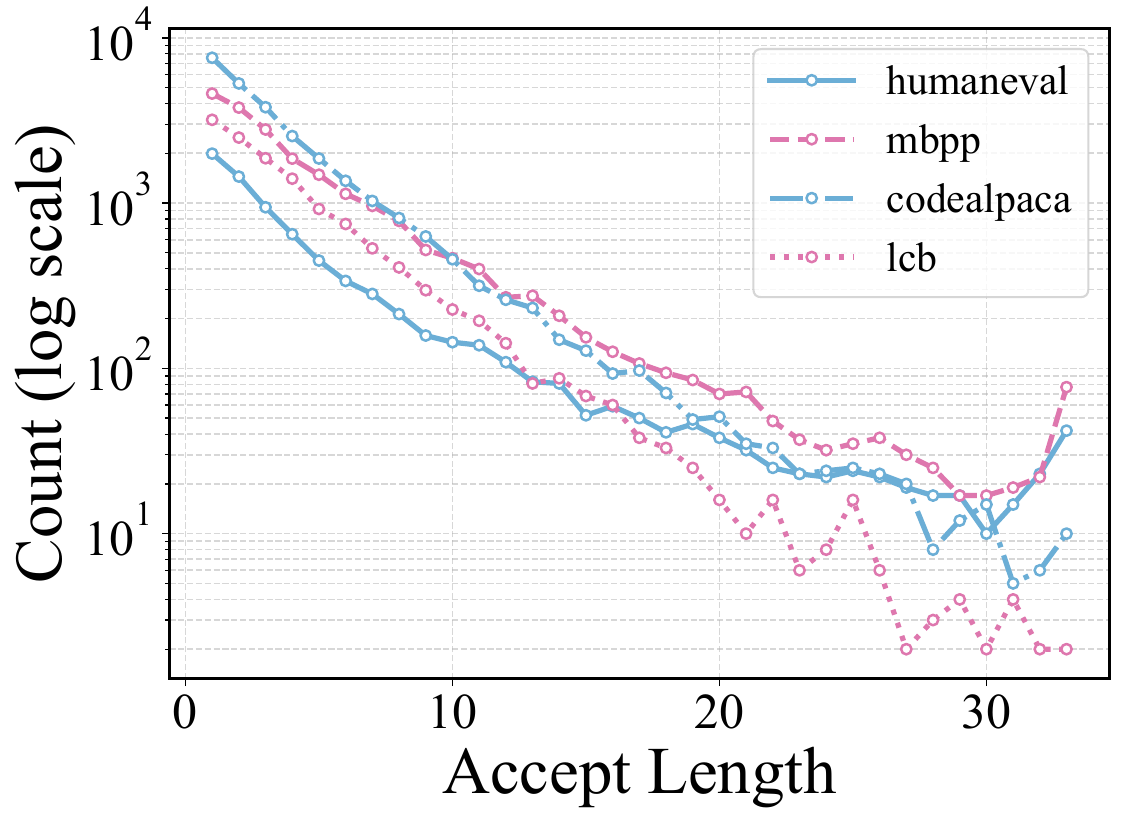}
    \caption{Qwen3-8B}
    \label{fig:sub2_line}
\end{subfigure}\hfill
\begin{subfigure}{0.32\textwidth}
    \centering
    \includegraphics[width=\linewidth]{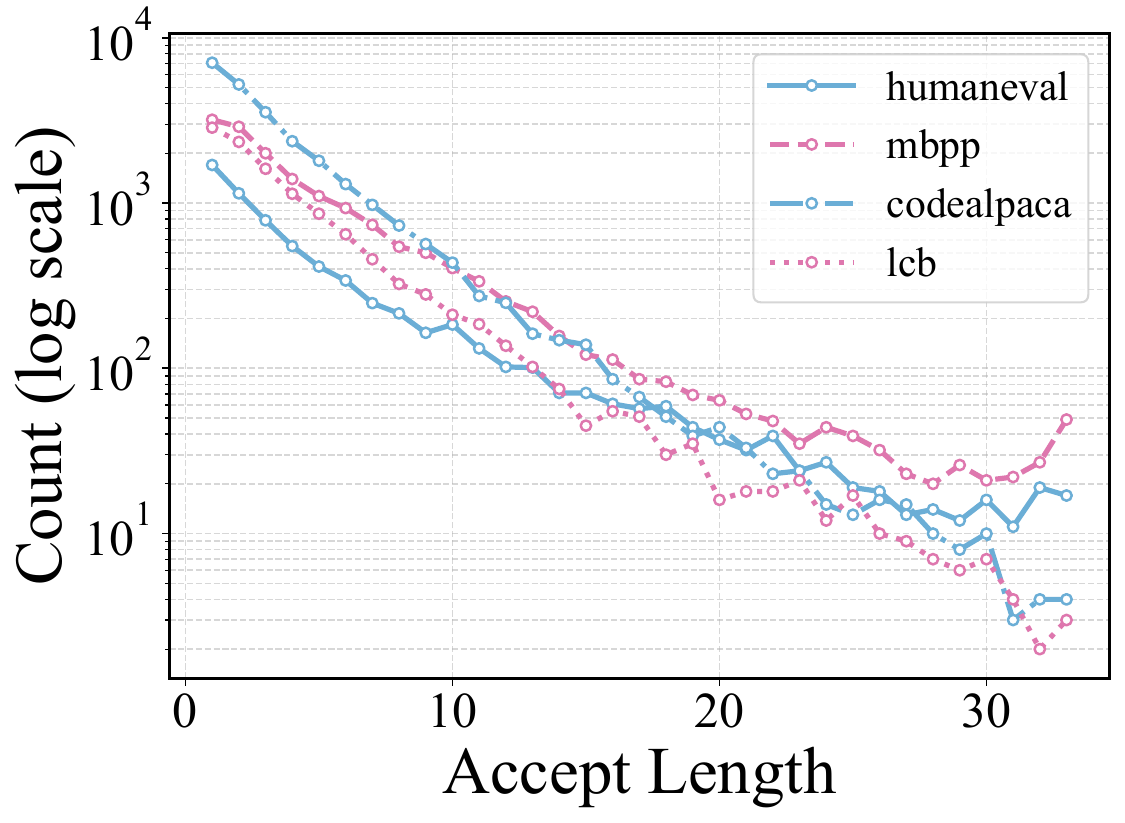}
    \caption{Qwen3-14B}
    \label{fig:sub3_line}
\end{subfigure}\hfill
\begin{subfigure}{0.32\textwidth}
    \centering
    \includegraphics[width=\linewidth]{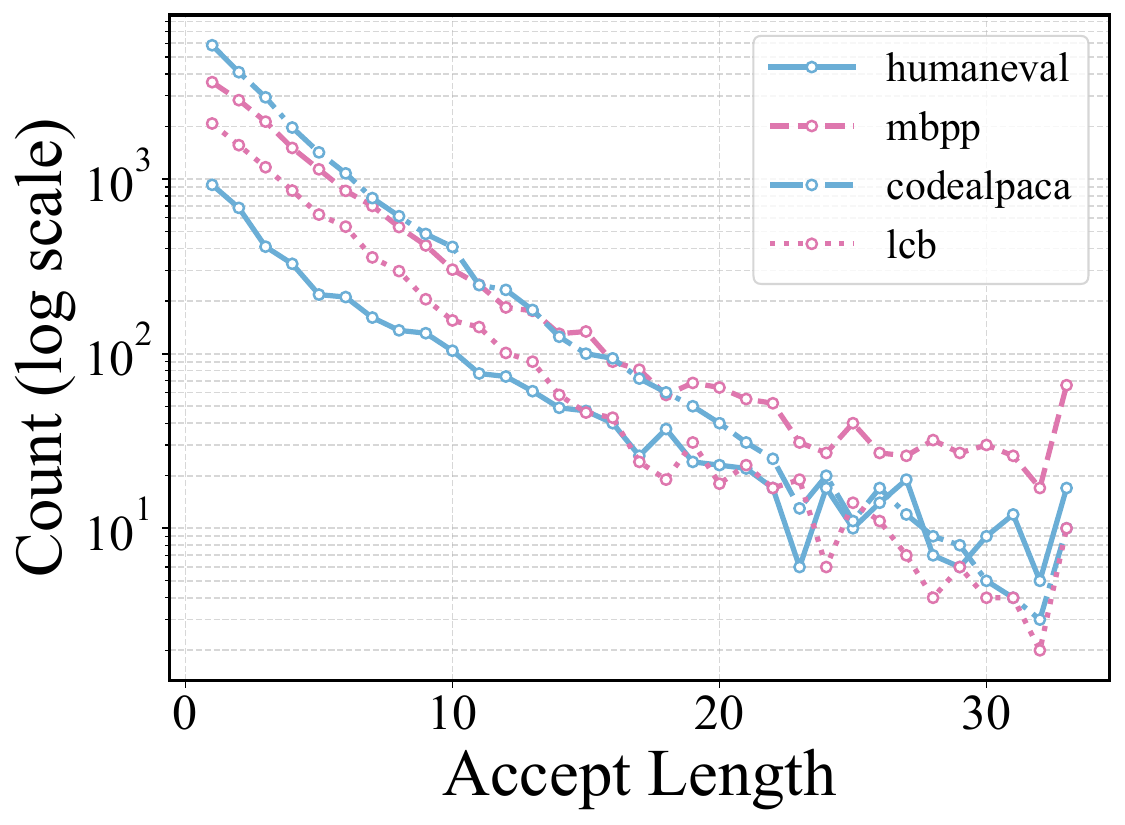}
    \caption{Qwen3-30B-A3B}
    \label{fig:sub4_line}
\end{subfigure}

\caption{
Token-length distribution of accuracy differences across four code benchmarks for Qwen3 models.
(a)--(c) correspond to Qwen3-8B, Qwen3-14B, and Qwen3-30B-A3B, respectively.
The x-axis denotes the number of tokens and the y-axis shows the frequency (log scale).
}
\label{fig:combined}

\end{figure}

\subsection{Sensitivity Analysis of block size}

As shown in Figure~\ref{fig:combined_acc}, the average acceptance length consistently increases with larger block sizes, while the growth rate gradually slows down, indicating a trade-off between longer blocks and the corresponding computation overhead. Furthermore, the scaling of the backbone models significantly enhances the acceptance behavior. Across block sizes from 4 to 32, Qwen3-14B improves acceptance length by \textbf{5\%--14\%} over Qwen3-8B, while Qwen3-30B-A3B delivers further gains of \textbf{14\%--33\%}. We attribute this improvement to stronger output determinism in larger base models, which enables DLLM to more effectively fit deterministic token patterns. Therefore, we expect our acceleration method to yield even greater benefits when deployed on larger-scale LLMs.

\begin{figure}[ht!] 
    
\begin{subfigure}{0.32\textwidth}
    \centering
    \includegraphics[width=\linewidth]{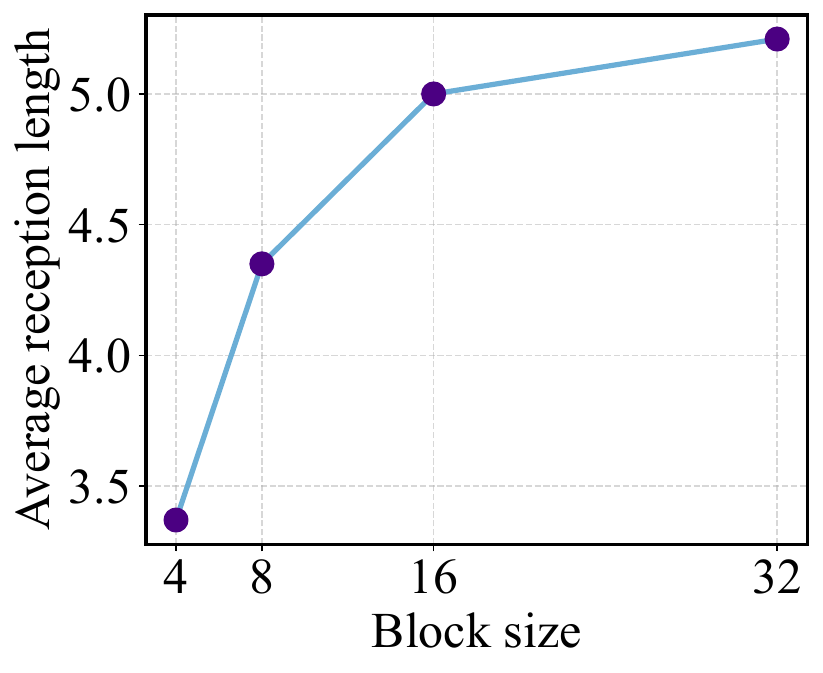}
    \caption{Qwen3-8B}
    \label{fig:sub2}
\end{subfigure}\hfill
\begin{subfigure}{0.32\textwidth}
    \centering
    \includegraphics[width=\linewidth]{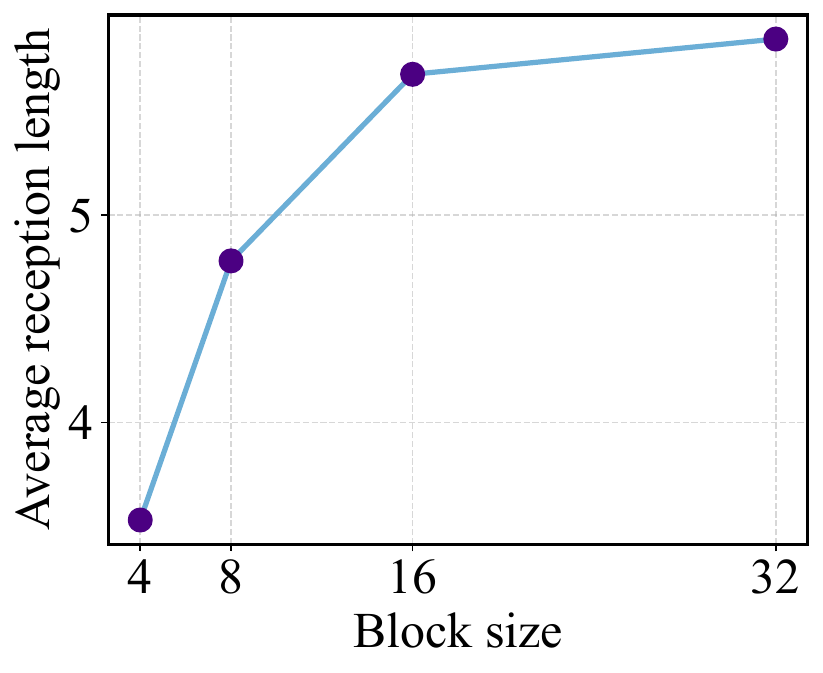}
    \caption{Qwen3-14B}
    \label{fig:sub3}
\end{subfigure}\hfill
\begin{subfigure}{0.32\textwidth}
    \centering
    \includegraphics[width=\linewidth]{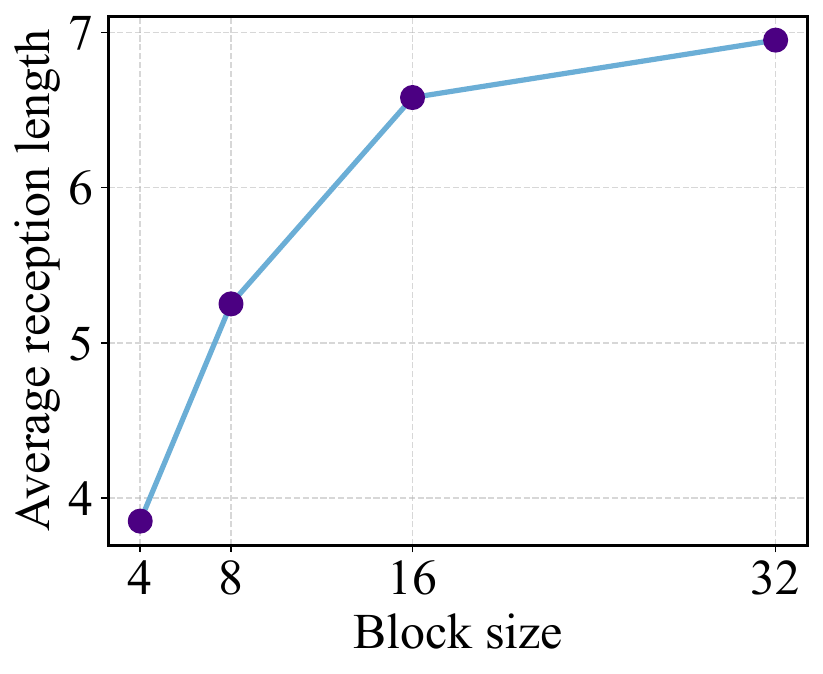}
    \caption{Qwen3-30B-A3B}
    \label{fig:sub4}
\end{subfigure}

\caption{
Block-size distribution of average acceptance lengths for Qwen3 models evaluated on the HumanEval benchmark.
Subfigures (a)--(c) correspond to Qwen3-8B, Qwen3-14B, and Qwen3-30B-A3B, respectively.
The x-axis represents the block size, while the y-axis denotes the average acceptance length.
}
\label{fig:combined_acc}

\end{figure}

\section{Additional Analysis of Drafter Size}
\label{app:drafter_size}
One potential concern is that the gains of DEER might stem from using a much larger drafter than prior speculative decoding methods such as EAGLE-3, rather than from the proposed drafting strategy itself. To examine this factor, we compare the parameter scales of the drafters used by EAGLE-3 and DEER under the same target backbones. The results are summarized in Table~\ref{tab:drafter_size_comparison}.

As shown in Table~\ref{tab:drafter_size_comparison}, the drafter sizes of EAGLE-3 and DEER are of the same order of magnitude (hundreds of millions of parameters) and remain well below the corresponding target model sizes. For the 8B backbone, DEER uses a slightly larger drafter than EAGLE-3 (470M vs.\ 400M), whereas for the 14B backbone EAGLE-3 employs an even larger drafter (610M vs.\ 470M). For the 30B backbone, DEER again uses a larger drafter (470M vs.\ 140M), but the drafter still accounts for only a small fraction of the 30B target. Under these matched-capacity regimes, DEER nevertheless achieves substantially longer average acceptance lengths (see Table~\ref{tab:model_performance_comparison_tem0}) and higher maximum accepted block sizes (see Table~\ref{tab:model_comparison_maxlenth}) than EAGLE-3. This indicates that our improvements primarily come from the discrete diffusion drafting mechanism and the proposed alignment pipeline, rather than from simply scaling up the drafter model.

\begin{table}[t]
    \centering
    \caption{Comparison of drafter sizes between EAGLE-3 and DEER under matched target models.}
    \label{tab:drafter_size_comparison}
    \begin{tabular}{lcccc}
        \toprule
        Target model & Target params (B) & Method   & Drafter type      & Drafter params (M) \\
        \midrule
        Qwen3-8B     & 8.0              & EAGLE-3  & AR draft head     & 400      \\
                     &                  & DEER     & discrete DLLM     & 470       \\
        \midrule
        Qwen3-14B    & 14.0             & EAGLE-3  & AR draft head     & 610     \\
                     &                  & DEER     & discrete DLLM     & 470      \\
        \midrule
        Qwen3-30B-A3B & 30.0            & EAGLE-3  & AR draft head     & 140     \\
                     &                  & DEER     & discrete DLLM     & 470                \\
        \bottomrule
    \end{tabular}
    
    \vspace{0.3em}
    {\footnotesize
    Note: Drafter parameter counts for EAGLE-3 are taken from the original paper. 
    DEER uses a 0.5B-parameter discrete diffusion model as the drafter for all target models.
    }
\end{table}

\section{Training Cost Comparison}
\label{app:training_time}

We further compare the fine-tuning cost of different speculative decoding drafters. In particular, we report the \emph{approximate} training time for Medusa, Hydra, EAGLE-3, and DEER on the Qwen3-8B backbone, and attempt to scale the same configurations to Qwen3-14B. The results are summarized in Table~\ref{tab:training_time}.

As shown in Table~\ref{tab:training_time}, all four methods can be trained on Qwen3-8B within a comparable range of GPU hours. However, when moving to the Qwen3-14B backbone, both Medusa and Hydra encounter out-of-memory (OOM) errors under their default publicly released configurations in our setup, even after standard tuning of batch size and sequence length. Consequently, we only report training times for EAGLE-3 and DEER on Qwen3-14B, and exclude Medusa and Hydra from 14B-scale efficiency comparisons.

\begin{table}[t]
    \centering
    \caption{Approximate fine-tuning time of different speculative decoding drafters on Qwen3 backbones.}
    \label{tab:training_time}
    \begin{tabular}{l l c c}
        \toprule
        Target model & Method & Training time (GPU hours) & Status \\
        \midrule
        Qwen3-8B  & Medusa & 768  & succeeds \\
                   & Hydra  & 1800 & succeeds \\
                   & EAGLE-3 & 696 & succeeds \\
                   & DEER   & 240  & succeeds \\
        \midrule
        Qwen3-14B & Medusa & --   & OOM (default config) \\
                   & Hydra  & --   & OOM (default config) \\
                   & EAGLE-3 & 1440 & succeeds \\
                   & DEER   & 240  & succeeds \\
        \bottomrule
    \end{tabular}
    
    \vspace{0.3em}
    {\footnotesize
    Note: GPU hours are approximate wall-clock measurements under our training setup 
    using the official or default configurations released for each method. For Qwen3-14B, 
    Medusa and Hydra run out of memory (OOM) under their default settings in our 
    environment, so we only report training times for EAGLE-3 and DEER.
    }
\end{table}

\section{Correctness Proof of the DEER Speculative Decoding Algorithm}
\label{app:lossless}

In this section, we provide a formal proof that DEER is \emph{lossless}, i.e., it produces exactly the same output distribution as sampling directly from the target autoregressive (AR) model $p_{\mathrm{AR}}$.

\subsection{A One-Step Draft--Then--Verify Lemma}

We first analyze a single decoding position. Fix a time step $j+i$ and a realized prefix $\mathbf{x}_{1:j+i-1}$. For convenience, define the target conditional distribution at position $j+i$ as
\[
p(x)
\;\triangleq\;
p_{\mathrm{AR}}\bigl(x_{j+i}=x \mid \mathbf{x}_{1:j+i-1}\bigr),
\]
and let
\[
P(x)
\;\triangleq\;
P\bigl(x_{j+i}=x\bigr)
\]
be an \emph{arbitrary} proposal distribution over the same vocabulary. We only require the standard support condition
\begin{equation}
\label{eq:support-condition-app}
p(x) > 0 \;\Rightarrow\; P(x) > 0
\quad\text{for all } x,
\end{equation}
i.e., the proposal never assigns zero probability where the target is positive.

Following the classical speculative decoding analysis~\citep{SPS}, we define the pointwise overlap
\[
m(x)
\;\triangleq\;
\min\bigl(p(x), P(x)\bigr),
\]
its total mass
\[
\gamma
\;\triangleq\;
\sum_x m(x),
\]
and the \emph{residual} distribution
\begin{equation}
\label{eq:residual-dist-app}
p_{\mathrm{res}}(x)
\;\triangleq\;
\frac{p(x) - m(x)}{1 - \gamma}
\quad\text{for } 1-\gamma > 0,
\end{equation}
with any arbitrary definition (e.g., $p_{\mathrm{res}}=p$) in the degenerate case $\gamma = 1$ (note that then the residual branch is never used).

\begin{lemma}[One-step lossless speculative sampling]
\label{lem:one-step-app}
Consider the following sampling procedure for the token at position $j+i$ given $\mathbf{x}_{1:j+i-1}$:
\begin{enumerate}
    \item Draw a draft token $Y \sim P(\cdot)$.
    \item Define the acceptance probability
    \begin{equation}
    \label{eq:alpha-min-app}
        \alpha(Y)
        \;=\;
        \min\!\Bigl(1, \tfrac{p(Y)}{P(Y)}\Bigr)
        \;=\;
        \frac{m(Y)}{P(Y)}.
    \end{equation}
    \item Draw $U \sim \mathrm{Uniform}(0,1)$ independently.
    \begin{itemize}
        \item If $U \le \alpha(Y)$, \emph{accept} the draft and set $Z = Y$.
        \item Otherwise, \emph{reject} the draft and set $Z \sim p_{\mathrm{res}}(\cdot)$ as in~\eqref{eq:residual-dist-app}.
    \end{itemize}
\end{enumerate}
Then, for every token $a$ in the vocabulary,
\[
\mathbb{P}\bigl[Z = a \mid \mathbf{x}_{1:j+i-1}\bigr] \;=\; p(a),
\]
i.e., the final token $Z$ has exactly the target distribution $p(\cdot)$ at position $j+i$.
\end{lemma}

\begin{proof}
Fix any token $a$. The probability that the procedure outputs $a$ can be decomposed into two disjoint events:
(i) $a$ is drafted and accepted, and
(ii) the draft is rejected and $a$ is obtained from the residual distribution:
\begin{align*}
\mathbb{P}[Z = a]
&=
\underbrace{\mathbb{P}[\text{accept and } Z=a]}_{\text{draft accepted}}
+
\underbrace{\mathbb{P}[\text{reject}]\,\mathbb{P}[Z=a \mid \text{reject}]}_{\text{draft rejected}}.
\end{align*}
For the first term, using~\eqref{eq:alpha-min-app} we have
\[
\mathbb{P}[\text{accept and } Z=a]
=
\mathbb{P}[Y=a]\,\alpha(a)
=
P(a)\,\frac{m(a)}{P(a)}
=
m(a).
\]
For the rejection probability,
\[
\mathbb{P}[\text{reject}]
=
1 - \sum_x \mathbb{P}[Y=x]\,\alpha(x)
=
1 - \sum_x P(x)\,\frac{m(x)}{P(x)}
=
1 - \sum_x m(x)
=
1 - \gamma.
\]
Conditioned on rejection, $Z$ is drawn from $p_{\mathrm{res}}$ in~\eqref{eq:residual-dist-app}, so
\[
\mathbb{P}[Z=a \mid \text{reject}]
=
p_{\mathrm{res}}(a)
=
\frac{p(a) - m(a)}{1-\gamma}.
\]
Putting everything together,
\begin{align*}
\mathbb{P}[Z = a]
&=
m(a)
+
(1-\gamma)\,\frac{p(a) - m(a)}{1-\gamma}
\\
&=
m(a) + p(a) - m(a)
\\
&=
p(a),
\end{align*}
which proves the claim.
\end{proof}

Lemma~\ref{lem:one-step-app} shows that for \emph{any} proposal $P(x_{j+i})$ satisfying the support condition~\eqref{eq:support-condition-app}, the draft--then--verify step with acceptance probability $\alpha(\cdot)$ and residual distribution $p_{\mathrm{res}}(\cdot)$ is exactly lossless.

\subsection{Instantiating the Proposal with the DEER Drafter}

We now instantiate Lemma~\ref{lem:one-step-app} with the proposal used in DEER. At position $j+i$, given the prefix $\mathbf{x}_{1:j+i-1}$, the target conditional is
\[
p(x)
=
p_{\mathrm{AR}}\bigl(x_{j+i}=x \mid \mathbf{x}_{1:j+i-1}\bigr),
\]
while the DEER drafter proposes tokens according to the diffusion model
\[
P(x)
=
q_\theta\bigl(x_{j+i}=x \mid \mathbf{x}_{1:j}\bigr).
\]
Our training in Section~3.1 ensures that whenever
\[
p_{\mathrm{AR}}\bigl(x_{j+i}=x \mid \mathbf{x}_{1:j+i-1}\bigr) > 0,
\]
we also have
\[
q_\theta\bigl(x_{j+i}=x \mid \mathbf{x}_{1:j}\bigr) > 0,
\]
so the support condition~\eqref{eq:support-condition-app} holds.

Plugging these $p$ and $P$ into Lemma~\ref{lem:one-step-app}, the acceptance probability takes the familiar form
\begin{equation}
\label{eq:alpha-deer-app}
\alpha_i
=
\min\!\left(
1,\;
\frac{
p_{\mathrm{AR}}(\hat{y}_{j+i} \mid \mathbf{x}_{1:j+i-1})
}{
q_\theta(\hat{y}_{j+i} \mid \mathbf{x}_{1:j})
}
\right),
\end{equation}
which is exactly Eq.~\eqref{eq:accept} in the main text. When the draft token at position $j+i$ is rejected, Lemma~\ref{lem:one-step-app} tells us that it should be replaced by a sample from the corresponding residual distribution
\[
p_{\mathrm{res}, j+i}(x)
=
\frac{
p_{\mathrm{AR}}(x \mid \mathbf{x}_{1:j+i-1})
-
\min\!\Bigl(
p_{\mathrm{AR}}(x \mid \mathbf{x}_{1:j+i-1}),
q_\theta(x \mid \mathbf{x}_{1:j})
\Bigr)
}{
1 - \sum_{x'}
\min\!\Bigl(
p_{\mathrm{AR}}(x' \mid \mathbf{x}_{1:j+i-1}),
q_\theta(x' \mid \mathbf{x}_{1:j})
\Bigr)
}.
\]
Conceptually, this replacement step corresponds to Eq.~\eqref{eq:resample} in the main text:
\begin{equation}
\hat{y}_{j+i}
\sim
p_{\mathrm{AR}}(\cdot \mid \mathbf{x}_{1:j+i-1}),
\end{equation}
together with the standard residual construction of speculative decoding~\citep{SPS}. Under this interpretation, Lemma~\ref{lem:one-step-app} directly implies that for every position $j+i$,
\[
\mathbb{P}\bigl[x_{j+i} = x \mid \mathbf{x}_{1:j+i-1}\bigr]
=
p_{\mathrm{AR}}\bigl(x_{j+i} = x \mid \mathbf{x}_{1:j+i-1}\bigr),
\]
i.e., the conditional distribution of the final token matches the AR model exactly.

\subsection{Sequence-Level Losslessness of DEER}

Finally, we extend the one-step result to the whole generated sequence.

\begin{theorem}[Losslessness of DEER decoding]
\label{thm:deer-lossless-app}
Let $p_{\mathrm{AR}}$ denote the target autoregressive model. Consider running DEER decoding with drafter $q_\theta$ and acceptance/resampling rules given by Eqs.~\eqref{eq:accept} and~\eqref{eq:resample}. Then for any finite sequence $\mathbf{x}_{1:T}$,
\[
\mathbb{P}_{\mathrm{DEER}}(\mathbf{x}_{1:T})
=
\prod_{t=1}^{T}
p_{\mathrm{AR}}\bigl(x_t \mid \mathbf{x}_{1:t-1}\bigr),
\]
i.e., DEER produces exactly the same joint distribution over outputs as direct autoregressive sampling from $p_{\mathrm{AR}}$.
\end{theorem}

\begin{proof}
We proceed by induction on $t$.

\textbf{Base case.}
For $t=1$, there is no history, and DEER samples from the target model by construction, so
\[
\mathbb{P}_{\mathrm{DEER}}(x_1) = p_{\mathrm{AR}}(x_1).
\]

\textbf{Inductive step.}
Assume that for some $t \ge 2$, the joint distribution over the first $t-1$ tokens generated by DEER matches that of $p_{\mathrm{AR}}$:
\[
\mathbb{P}_{\mathrm{DEER}}(\mathbf{x}_{1:t-1})
=
\prod_{s=1}^{t-1}
p_{\mathrm{AR}}(x_s \mid \mathbf{x}_{1:s-1}).
\]
Conditioned on any realized prefix $\mathbf{x}_{1:t-1}$, the one-step analysis in Lemma~\ref{lem:one-step-app} (instantiated as above) implies
\[
\mathbb{P}_{\mathrm{DEER}}(x_t \mid \mathbf{x}_{1:t-1})
=
p_{\mathrm{AR}}(x_t \mid \mathbf{x}_{1:t-1}).
\]
Therefore,
\begin{align*}
\mathbb{P}_{\mathrm{DEER}}(\mathbf{x}_{1:t})
&=
\mathbb{P}_{\mathrm{DEER}}(\mathbf{x}_{1:t-1})\,
\mathbb{P}_{\mathrm{DEER}}(x_t \mid \mathbf{x}_{1:t-1})
\\
&=
\Biggl(
\prod_{s=1}^{t-1}
p_{\mathrm{AR}}(x_s \mid \mathbf{x}_{1:s-1})
\Biggr)
p_{\mathrm{AR}}(x_t \mid \mathbf{x}_{1:t-1})
\\
&=
\prod_{s=1}^{t}
p_{\mathrm{AR}}(x_s \mid \mathbf{x}_{1:s-1}),
\end{align*}
which completes the induction.
\end{proof}

Theorem~\ref{thm:deer-lossless-app} formally establishes that DEER is a \emph{lossless} decoding scheme: it achieves acceleration by using the diffusion drafter $q_\theta$ to propose blocks of tokens, while provably preserving the exact output distribution of the target autoregressive model $p_{\mathrm{AR}}$.

\section{KV cache of DLLM}

At present, the community lacks mature support for diffusion language models (DLLMs) with KV caching in mainstream inference frameworks. Consequently, for batch size $B>1$ our implementation cannot yet be integrated with popular systems such as vLLM~\cite{vllm} and SGLang~\cite{SGLang}. Nevertheless, there has been rapid progress on enabling KV cache for DLLMs. Fast-dLLM~\cite{Fast-dLLM} is among the earliest attempts to explore KV caching for diffusion-based LMs, and its design is particularly well aligned with block diffusion architectures. More recently, dInfer~\cite{dInfer} proposes a complete and efficient KV cache mechanism for dLLMs together with a dedicated high-throughput inference engine. We expect these techniques to be gradually incorporated into mainstream inference frameworks such as vLLM and SGLang. Once such integration becomes available, DEER can naturally leverage these KV cache implementations and is expected to exhibit substantial advantages in batched inference scenarios.

\end{document}